\documentclass{article}

\usepackage[preprint]{neurips_2025}

\usepackage[utf8]{inputenc} %
\usepackage[T1]{fontenc}    %
\usepackage{hyperref}       %
\usepackage{url}            %
\usepackage{booktabs}       %
\usepackage{amsfonts}       %
\usepackage{nicefrac}       %
\usepackage{microtype}      %
\usepackage{xcolor}         %
\usepackage{threeparttable} %
\usepackage{graphicx}       %
\usepackage{wrapfig}

\usepackage[ruled,vlined]{algorithm2e}
\usepackage{xspace}
\usepackage{mathtools}
\usepackage[most]{tcolorbox}
\tcbset{
  bluebox/.style={
    colback=pblue!15,   %
    boxrule=0pt,          %
    arc=1mm,                %
    left=0pt, right=0pt, top=0pt, bottom=0pt,
    boxsep=1pt,
    enhanced,
  }
}
\usepackage{float}  %

\usepackage{notations/def}
\usepackage{notations/notations}

\newcommand{\algo}{\texttt{\textcolor{pblue}{SORL}}\xspace}
\newcommand{\algofull}{Scalable Offline Reinforcement Learning\xspace}
\newcommand{\algofullit}{\textit{Scalable Offline Reinforcement Learning}\xspace}

\newcommand{\dbtt}{M^{\text{BTT}}}
\newcommand{\dinf}{M^{\text{inf}}}
\newcommand{\dreg}{M^{\text{disc}}}
\newcommand{\fql}{\texttt{FQL}\xspace}
\definecolor{pblue}{HTML}{4bacc6}
\definecolor{porange}{HTML}{F79646}
\definecolor{ppurple}{HTML}{8064A2}

\title{Scaling Offline RL via\\ Efficient and Expressive Shortcut Models}

\author{%
  Nicolas Espinosa-Dice\\
  Cornell University \\
  \texttt{ne229@cornell.edu} \\
  \And
  Yiyi Zhang \\
  Cornell University \\
  \texttt{yz2364@cornell.edu} \\
  \And
  Yiding Chen \\
  Cornell University \\
  \texttt{yc2773@cornell.edu} \\
  \And
  Bradley Guo \\
  Cornell University \\
  \texttt{bzg4@cornell.edu} \\
  \And
  Owen Oertell \\
  Cornell University \\
  \texttt{ojo2@cornell.edu} \\
  \And
  Gokul Swamy \\
  Carnegie Mellon University \\
  \texttt{gswamy@andrew.cmu.edu} \\
  \And
  Kianté Brantley \\
  Harvard University \\
  \texttt{kdbrantley@harvard.edu} \\
  \And
  Wen Sun \\
  Cornell University \\
  \texttt{ws455@cornell.edu} \\
}

\hypersetup{
    pdftitle={Scaling Offline RL via Efficient and Expressive Shortcut Model}, %
    pdfauthor={Nicolas Espinosa-Dice, Yiyi Zhang, Yiding Chen, Bradley Guo, Owen Oertell, Gokul Swamy, Kiante Brantley, Wen Sun}, %
    pdfsubject={Reinforcement Learning}, %
    pdfkeywords={Offline Reinforcement Learning, Generative Models} %
}

\begin{document}

\maketitle

\begin{abstract}
Diffusion and flow models have emerged as powerful generative approaches capable of modeling diverse and multimodal behavior. However, applying these models to offline reinforcement learning (RL) remains challenging due to the iterative nature of their noise sampling processes, making policy optimization difficult. In this paper, we introduce \algofullit (\algo), a new offline RL algorithm that leverages shortcut models---a novel class of generative models---to scale both training and inference. \algo's policy can capture complex data distributions and can be trained simply and efficiently in a one-stage training procedure. At test time, \algo introduces both sequential and parallel inference scaling by using the learned $Q$-function as a verifier. We demonstrate that \algo achieves strong performance across a range of offline RL tasks and exhibits positive scaling behavior with increased test-time compute. We release the code at 
\textcolor{pblue}{\texttt{\href{https://nico-espinosadice.github.io/projects/sorl/}{\texttt{nico-espinosadice.github.io/projects/sorl}}}}.
\end{abstract}

\section{Introduction}
Offline reinforcement learning (RL) \citep{ernst2005tree, lange2012batch,levine2020offline} is a paradigm for using fixed datasets of interactions to train agents without online exploration. In this paper, we tackle the challenge of scaling offline RL, for which there are two core components: \textit{training} and \textit{inference}. 

In order to scale training, offline RL algorithms must be capable of handling larger, more diverse multi-modal datasets 
\citep{o2024open,rafailov2024d5rl,park2024ogbench,hussing2023robotic,gurtler2023benchmarking,bu2025agibot}. 
While standard Gaussian-based policy classes cannot model multi-modal data distributions, 
flow matching \citep{lipman2022flow,liu2022flow,albergo2023stochastic} 
and diffusion models \citep{sohl2015deep,ho2020denoising,song2021scorebased}
have emerged as powerful, highly expressive model classes capable of modeling complex data distributions. However, while generative models like diffusion and flow matching can model diverse offline data, applying them to offline RL is challenging due to their iterative noise sampling process, which makes policy optimization difficult, often requiring backpropagation through time or distillation of a larger model. Ultimately, we desire a training procedure that can efficiently train an expressive model class.

The second core component in scaling offline RL is inference. During inference, we desire both efficiency and precision: the agent must be able to make decisions rapidly (e.g. autonomous vehicles) but also take precise actions (e.g. surgical robots). Inspired by the recent work in test-time scaling of large language models
\citep{wei2023chainofthoughtpromptingelicitsreasoning,deepseekai2025deepseekr1incentivizingreasoningcapability,gui2024bonbon,brown2024largelanguagemonkeysscaling,muennighoff2025s1simpletesttimescaling,madaan2023selfrefineiterativerefinementselffeedback,qu2024recursiveintrospectionteachinglanguage,qin2025backtrack}, we investigate how test-time scaling can be applied in offline RL with generative models. That is, we desire an approach that can perform inference efficiently---avoiding the slow, many-step generation process of diffusion models---but can also leverage additional test-time compute when available. 

Unfortunately, recent work in offline RL fails to achieve both of the desiderata necessary to scale offline RL. Distillation-based approaches to offline RL with generative models \citep{ding2023consistency, chen2023score, chen2024diffusion, park2025flow} avoid extensive backpropagation through time during training. However, they may require more complex, two-stage training procedures (e.g. propagating through teacher/student networks), over which error compounds. They also have limited scaling of inference steps, thus losing expressivity when compared to multi-step generative models \citep{frans2024one}. Alternatively, diffusion-based approaches that leverage backpropagation through time \citep{wang2022diffusion,he2023diffcps,zhang2024entropy,ada2024diffusion} may learn policies with greater expressivity. However, diffusion models have slow inference, requiring a larger number of steps to generate high quality outputs \citep{frans2024one}.

At a high-level, offline RL algorithms that employ generative models struggle with the following tradeoff. In order to achieve training efficiency, we want to avoid performing many steps of the iterative noise sampling process during policy optimization, which often requires backpropagating many steps through time. However, modeling complex distributions in the offline data, such as diverse or multi-modal data, may require a larger number of discretization steps to allow for maximum expressivity. Finally, at inference-time, we desire the ability to both generate actions efficiently via a few-step sampling procedure (e.g. in robot locomotion settings) \textit{and} leverage additional test-time compute when available (e.g. in robot manipulation settings).

In this paper, we tackle the question of how to achieve efficient training, while maintaining expressivity, \textit{and} scale with greater inference-time compute. We introduce \algofullit (\algo): a simple, efficient one-stage training procedure for expressive policies. Our key insight is to leverage \textit{shortcut models} \citep{frans2024one}, a novel class of generative models, in order to incorporate \textit{self-consistency} into the training process, thus allowing the policy to generate high quality samples under \textit{any} inference budget. \algo's self consistency property allows us to vary the number of denoising steps used for the three core components in offline RL: policy optimization (i.e. the number of backpropagation through time steps), regularization to offline data (i.e. the total number of discretization steps), and inference (i.e. the number of inference steps). We can use fewer steps for policy optimization, thus making training efficient, while performing inference under varying inference budgets, depending on the desired inference-time compute budget. We incorporate shortcut models into regularized actor-critic algorithm and make the following contributions:
\begin{enumerate}
    \item We introduce \textbf{\algo, an efficient, one-stage training procedure for expressive policies that can perform inference under \textit{any} compute budget}, including one-step inference. Despite its efficiency, \algo's policy is still able to capture complex, multi-modal data distributions.
    \item Through a novel theoretical analysis of shortcut models, \textbf{we prove that \algo's training objective regularizes \algo's policy to the behavior of the offline data}. 
    \item Empirically, \textbf{\algo achieves the best performance against 10 baselines on a range of diverse tasks in offline RL}. 
    \item At test-time, \textbf{\algo can scale with greater inference-time compute by increasing the number of inference steps (i.e. sequential scaling) and performing best-of-$N$ sampling (i.e. parallel scaling)}. We show empirically that \algo can make up for a smaller training-time compute budget with a greater inference-time compute budget. \algo can also generalize to more inference steps at test time than the number of steps optimized during training.
\end{enumerate}

\section{Background}
\label{sec:background}
\paragraph{Markov Decision Process.}
We consider an infinite-horizon Markov Decision Process (MDP) \citep{puterman2014markov},
$\mathcal{M} = \langle \Xcal, \Acal, P, R, \gamma, \mu \rangle$. 
$\Xcal$ and $\Acal$ are the state space and action space, respectively.
$P: \Xcal \times \Acal \rightarrow \Delta (\Xcal)$ is the transition function,
$R: \Xcal \times \Acal \rightarrow [0, 1]$ is the reward function,
and $\gamma$ is the discount factor.
$\mu \in \Delta (\Xcal)$ is the starting state distribution.
Let $\Pi=\{\pi: \Xcal \rightarrow \Delta(\Acal)\}$ be the class of stationary policies.
We define the state-action value function $Q(x, a): \Xcal \times \Acal \rightarrow \RR$  as
$Q^\pi(x, a) := \EE_\pi \left [ \sum_{i=1}^{\infty} R(x_i, a_i) \mid x_i = x, a_i = a \right ]$.
We define the state visitation distribution generated by a policy $\pi$ to be
$d^\pi_\mu := (1-\gamma) \EE_{x_0 \sim \mu} \left [ \sum_{i=0}^{\infty}\gamma^i \Pr_i^{\pi}(x \mid x_0) \right ]$.
In the offline RL setting, we assume access to a fixed dataset $\Dcal = \{(x, a, r, x')^{(j)}\}_{j \in 1, \ldots, n}$, such that the data was generated $(x,a,x') \sim d^{\pi_B}$ by some behavior policy $\pi_B$.

\paragraph{Flow Matching.} We define flow matching \citep{lipman2022flow,liu2022flow} as follows. 
Let $p^{\star} \in \Delta(\RR^d)$ be the target data distribution in $d$-dimensional Euclidean space. Let $\bar z_0 \sim \Ncal(0,I)$ and $\bar z_1\sim p^{\star}$ be two independent random variables. We define $\bar z_t$ to be the linear-interpolation between the $\bar z_0$ and $\bar z_1$:
\begin{align}
    \bar z_t := t\bar z_1 + (1-t) \bar z_0, \quad 0\le t\le 1.
\end{align}
$\{\bar z_t\}_{t\in[0,1]}$ is fully determined by the start point $\bar z_0$ and end point $\bar z_1$. We use $\{p_t\}_{t\in[0,1]}$ to denote the sequence of marginal distributions of $\bar z_t$'s, where $p_0 = \Ncal(0,I)$ and $p_1 = p^{\star}$. 

The \emph{drift function} $v_t(\cdot): \RR^d \rightarrow \RR^d$ is defined to be the solution to the following least square regression:
\begin{align}
\label{eq:flow-def-opt}
    \min_f \int_0^1 \EE_{\bar z_0\sim \Ncal(0,I), \bar z_1 \sim p^{\star}}\left[\left\|\bar z_1 - \bar z_0 - f(t\bar z_1 + (1-t)\bar z_0, t)\right\|_2^2\right] \d t.
\end{align}
We use $f^{\star}$ to denote the solution to the optimization problem in Equation~\ref{eq:flow-def-opt} and define:
\begin{align*}
    v_t(z) := f^{\star}(z,t),\quad \forall z\in\RR^d, t\in[0,1].
\end{align*}
The drift function $v_t(\cdot)$ induces an ordinary differential equation (ODE)
\begin{align}
    \label{eq:flow-ode}
    \frac{\d}{\d t}z_t = v_t (z_t).
\end{align}
An appealing property of Equation~\ref{eq:flow-ode} is: if the initial $z_0$ is drawn from Gaussian distribution $\Ncal(0,I)$, then the marginal distribution of $z_t$, the solution to the ODE, is exactly $p_t$, the marginal distribution of the linear-interpolation process \citep{liu2022flow}. In practice, one can learn a drift function $\hat v_t(\cdot)$ by directly optimizing Equation~\ref{eq:flow-def-opt}, starting from $z_0 \sim \Ncal(0,I)$, and solve the ODE: $\d z_t = \hat v_t (z_t) \d t$. The solution $z_1$ is an approximate sample from the target distribution $p^{\star}$.

\paragraph{Shortcut Models.}
Flow matching typically rely on small, incremental steps, which can be computationally expensive at inference time. Shortcut models improve inference efficiency by learning to take larger steps along the flow trajectory \citep{frans2024one}. The key insight of \citet{frans2024one} is to condition the model not only on the timestep $t$, as in standard flow matching, but also on a step size $h$. 

The shortcut model $s(z_t, t, h)$ predicts the normalized direction from $z_t$ towards the correct next point $z_{t+h}$, such that
\begin{equation}
    z_t + s(z_t, t, h)h \approx z_{t+h}.
\end{equation}
The objective is to learn a shortcut model $s_\theta(z_t, t, h)$ for all combinations of $z_t, t, h$. The loss function is given by
\begin{equation}
    \Lcal^S(\theta) = 
    \EE_{z_0 \sim \Ncal(0,I), \atop z_1 \sim D, (t, h) \sim p(t, h)}
    \Big [
    \big \Vert 
    \underbrace{s_\theta(z_t, t, 1/M) - (z_1 - z_0)}
    _{\texttt{Flow Matching}}
    \big \Vert^2
    + 
    \big \Vert 
    \underbrace{s_\theta(z_t, t, 2h) - s_{\text{target}}}
    _{\texttt{Self-Consistency}}
    \big \Vert^2
    \Big ],
\end{equation}
where $s_{\text{target}} = s_\theta(z_t, t, h)/2 + s_\theta(z'_{t+h}, t, h)/2$ and $z'_{t+h} = z_t + s_\theta(z_t, t, h)h$. $h$ is sampled uniformly from the set of powers of 2 between 1 and $M$, the maximum number of discretization steps, and $t$ is sampled uniformly between 0 and 1, so $p(h, t) = \text{Unif}\big (\{2^k\}_{k=0}^{\lfloor \log_2 M \rfloor} \big ) \times \text{Unif}(0, 1)$.\footnote{Note that the interval [0, 1] is discretized into $M$ steps.}

The first component of the loss is standard flow matching which ensures that when $h=1/M$, the smallest step size, the model recovers the target direction $z_1 - z_0$.  The second component, self-consistency, encourages the model to produce consistent behavior across step sizes. Intuitively, it ensures that one large \textit{jump} of size $2h$ is equivalent to two combined \textit{jumps} of size $h$. Thus, the model learns to take larger jumps with more efficient, fewer-step inference procedures \citep{frans2024one}.

\section{Scaling Offline Reinforcement Learning}
At a high-level, offline RL aims to optimize a policy subject to some regularization constraint, which we can formulate as
\begin{equation}
    \label{eq:policy-opt-w-reg}
    \argmax_{\pi \in \Pi} 
    \underbrace{J_\mathcal{D}(\pi)}_{\texttt{Policy Optimization\vphantom{Regularization}}}
    -
    \underbrace{\alpha R(\pi, \pi_B)}_{\texttt{Regularization\vphantom{Policy Optimization}}}
\end{equation}
where $J_\Dcal(\pi)$ is the expected return over offline dataset $\Dcal$, $\pi_B$ is the offline data generating policy, and $R(\pi, \pi_B)$ is a regularization term (e.g. a divergence measure between $\pi$ and $\pi_B$). 

The regularization constraint can be implemented through a behavioral cloning (BC)-style loss \citep{wu2019behavior}, such as in the behavior-regularized actor-critic formulation \citep{wu2019behavior,fujimoto2021minimalist,tarasov2023revisiting,park2025flow} 
\begin{equation}
    \label{eq:policy-opt-w-kl-reg}
    \argmax_{\pi \in \Pi} 
    \EE_{x, a \sim \Dcal} \Big [ \EE_{a^\pi \sim \pi_\theta(\cdot \mid x)} \Big [ 
    \underbrace{Q_{\phi}(x, a^\pi)}
    _{\mathclap{\texttt{Q Loss}}}
    - 
    \underbrace{\alpha \log \pi_{\theta} (a \mid x)}
    _{\mathclap{\texttt{BC Loss}}}
    \Big ] \Big ]
\end{equation}
where the $Q_\phi$ function is trained via minimizing the Bellman error. In order to incorporate generative models, the BC loss can be replaced with score matching or flow matching \citep{wang2022diffusion, park2024ogbench}. However, the core challenge of performing offline RL with generative models is the policy optimization component, due to the iterative nature of the noise sampling process. 

\begin{algorithm}[t]
\caption{\algofull (\algo)}
\label{alg:training}
\KwData{Offline dataset $\mathcal{D}$}
\DontPrintSemicolon
\SetArgSty{textnormal}
\SetKwComment{tcp}{\# }{}      %
\SetKwComment{sectioncomment}{}{}  %
\While{not converged}{
    Sample $(x, a^1, x', r) \sim \mathcal{D},\quad a^0 \sim \mathcal{N}(0, I),\quad (h,t) \sim p(h,t)$ \tcp*[r]{Parallelize batch}
    $a^t \gets (1 - t)a^0 + ta^1$ \tcp*[r]{Noise action}
    \,
    
    \begin{tcolorbox}[bluebox,width=0.93\linewidth]
    \sectioncomment{$\triangleright$ \color{pblue} Q Update}
    \For{all batch elements}{
        $m \sim \text{Unif} \{1, \ldots, \dbtt\}$ \tcp*[r]{Choose number of inference steps}
        $a^{\pi} \sim \pi_\theta(\cdot \mid x, m)$ \tcp*[r]{Sample action}
    }

    \,
    
    \sectioncomment{$\triangleright$ \color{pblue} Self-Consistency}
    \For{all batch elements}{
        $s_t \gets s_{\theta}(a^t, t, h \mid x)$ \tcp*[r]{First small step}
        $a^{t+h} \gets a^t + s_t h$ \tcp*[r]{Follow ODE}
        $s_{t+h} \gets s_{\theta}(a^{t+h}, t + h, h \mid x)$ \tcp*[r]{Second small step}
        $s_{\text{target}} \gets \text{stopgrad}(s_t + s_{t+h} \mid x)/2$ \tcp*[r]{Self-consistency target}
    }

    \,
    
    \sectioncomment{$\triangleright$ \color{pblue} Flow Matching}
    \For{all batch elements}{
        $h \gets 1/\dreg$ \tcp*[r]{Use smallest step size}
        $s_{\text{target}} \gets a^1 - a^0$ \tcp*[r]{Flow-matching target}
    }

    \,

    $\theta \gets \nabla_{\theta} \Vert s_{\theta}(a^t, t, 2h \mid x) - s_{\text{target}}\Vert^2 + \nabla_{\theta} Q_\phi(x, a^{\pi})$ \tcp*[r]{Update actor}

    \end{tcolorbox}

    \sectioncomment{$\triangleright$ \color{pblue} Critic Update}
    \For{all batch elements}{
        $m \sim \text{Unif} \{1, \ldots, \dbtt \}$ \tcp*[r]{Choose number of inference steps}
        $a^{\pi}_{x'} \sim \pi_\theta(\cdot \mid x', m)$ \tcp*[r]{Sample action}
    }

    \,
    
    $\phi \gets \nabla_{\phi} \left( Q_\phi (x, a^1) - r - \gamma Q_{\phi}(x', a^{\pi}_{x'}) \right)^2$ \tcp*[r]{Update critic}
}
\end{algorithm}

\subsection{\algofull (\algo)}
\paragraph{Motivation.}
In order to scale offline RL, we desire an efficiently trained expressive policy that is scalable under any inference budget.
Our key insight is that, by incorporating self-consistency into training, \algo can vary the number of denoising steps used during policy optimization (i.e. the number of backpropagation through time steps), regularization to offline data (i.e. the total number of discretization steps), and inference (i.e. the number of inference steps)---thus enabling both efficient training of highly expressive policy classes \textit{and} inference-time scaling. Unlike two-stage methods, which take existing diffusion models and later distill one-step capabilities into them, \algo is a unified model in which varying-step inference is learned by a single network in one training run.\footnote{Note that directly regularizing to the empirical offline data, as done via flow matching in \algo, is preferable to regularizing with respect to a learned behavior cloning (BC) policy, as done by \fql, when the underlying distribution class is unknown. This is because, in the nonparametric setting, the empirical distribution is a statistically consistent estimator, and it is minimax rate-optimal under common metrics like total variation and 2-Wasserstein distance.}

\paragraph{Policy Class and Inference Procedure.} We model our policy by the shortcut function $s_\theta$, and sample actions via the Euler method with the shortcut function $s_\theta$. The full procedure is presented in Algorithm \ref{alg:sampling}. Slightly overloading notation, we condition on the number of inference steps $m$ when generating actions from the policy (i.e. $a \sim \pi_\theta(\cdot \mid x, m)$). At test-time, we sample actions using $\dinf$ inference steps. Note that, since the inference process corresponds to solving a deterministic ODE, which is approximated using the Euler method, we can perform backpropagation through time on actions sampled via Algorithm \ref{alg:sampling} during training.

\paragraph{Actor Loss.} We present \algo's full training procedure in Algorithm \ref{alg:training}. There are three components to \algo's training: the Q update, the regularization to offline data, and the self-consistency:
\begin{equation}
    \label{eq:training-loss}
    \Lcal_\pi(\theta)
    = 
    \underbrace{\Lcal_{\text{QL}}(\theta)}
    _{\mathclap{
    \texttt{\color{black}Q Loss\vphantom{Flow Matching LossSelf-Consistency Loss}}}}
    \hspace{1em}
    \hspace{1em}
    + 
    \hspace{1em}
    \hspace{1em}
    \hspace{1em}
    \underbrace{\Lcal_{\text{FM}}(\theta)}
    _{\mathclap{\color{black}\texttt{Flow Matching Loss\vphantom{Q LossSelf-Consistency Loss}}}}
    \hspace{1em}
    \hspace{1em}
    \hspace{1em}
    +
    \hspace{1em}
    \hspace{1em}
    \underbrace{\Lcal_{\text{SC}}(\theta)}
    _{\mathclap{\color{black}\texttt{Self-Consistency Loss\vphantom{Flow Matching LossQ Loss}}}}
\end{equation}

\textit{\textcolor{black}{(1) Q Update.}} For the Q update, we first sample actions via the inference procedure in Algorithm \ref{alg:sampling}, using a maximum of $\dbtt$ steps (i.e. backpropagating $\dbtt$ steps through time). The Q loss is computed with respect to this action:
\begin{equation}
    \label{eq:q-loss}
    \Lcal_{\text{QL}}(\theta) = \EE_{x  \sim \Dcal} \EE_{a^\pi \sim \pi_\theta(\cdot \mid x)} \left [ -Q_\phi(x, a^{\pi}) \right ]
\end{equation} 
Since $\dbtt$ typically is small---we experiment with $\dbtt=1, 2, 4, 8$---we can backpropagate the $Q$ loss efficiently, without needing importance weighting or classifier gradients. In other words, even though sampling $a^\pi \sim \pi_{\theta}$ may involve multi-step generations, $\nabla_\theta \Lcal_{\text{QL}}$ is still computable. Additionally, rather than use a fixed number of steps for sampling actions, we sample the number of steps uniformly from the set of powers of 2 between $1$ and $\dbtt$.\footnote{In other words, we sample $m \sim \text{Unif}\{2^k\}_{k=0}^{\lfloor \log_2 \dbtt \rfloor}$.}

\textit{\textcolor{black}{(2) Offline Data Regularization.}} We add a BC-style loss that serves as the regularization to offline data, which we implement via flow matching on the offline data, using $\dreg$ discretization steps. $a^0$ represents a fully noised action (i.e. noise sampled from a Gaussian), and $a^1$ represents a real action (i.e. actions sampled from the offline data $\Dcal$). The flow matching loss is thus:
\begin{equation}
    \Lcal_{\text{FM}}(\theta) = 
    \EE_{x, a^1 \sim \Dcal, a^0 \sim \Ncal, \atop h \sim p(h, t)} 
    \Big [ \Big \Vert 
    \underbrace{s_\theta(a^t, t, 1/\dreg \mid x)}
    _{\mathclap{\texttt{Velocity Prediction}}}
    \hspace{0.5em}
    - 
    \hspace{0.5em}
    \underbrace{(a^1 - a^0)}
    _{\mathclap{\texttt{Velocity Target}}}
    \Big \Vert^2 \Big ]
\end{equation}

\textit{\textcolor{black}{(3) Self-Consistency.}} We add a self-consistency loss to ensure that bigger jumps (e.g. the shortcut of an $m$-step procedure) are consistent with smaller jumps (e.g. the shortcut of a $2m$-step procedure):
\begin{equation}
    \Lcal_\text{SC}(\theta) = 
    \EE_{x, a^1 \sim \Dcal, \atop a^0 \sim \Ncal, (t, h) \sim p(t, h)} 
    \Big [ \Big \Vert 
    \underbrace{s_\theta(a^t, t, 2h \mid x)}
    _{\mathclap{\texttt{1 Double-Step}}}
    \hspace{0.5em}
    - 
    \hspace{0.5em}
    \underbrace{s_{\text{target}}}
    _{\mathclap{\texttt{2 Single-Steps}}}
    \Big \Vert^2 \Big ]
\end{equation}
where 
\begin{equation}
    s_{\text{target}} = 
    \underbrace{s_\theta(a^t, t, h \mid x)/2}
    _{\mathclap{\texttt{1st Single-Step}}}
    + 
    \underbrace{s_\theta(a^{t+h}, t, h \mid x)/2,}
    _{\mathclap{\texttt{2nd Single-Step}}}
\end{equation}
and 
\begin{equation}
    \underbrace{\vphantom{s_\theta(a^t, t, h)} a^{t+h}}_{\texttt{2nd Step's Action}} 
    \!\!=\!\!
    \underbrace{\vphantom{s_\theta(a^t, t, h)} a^t}_{\texttt{1st Step's Action}}
    + 
    \underbrace{\vphantom{s_\theta(a^t, t, h)} h}_{\texttt{Step Size}}
    ~
    \underbrace{s_\theta(a^t, t, h)}_{\texttt{Normalized Direction}}
\end{equation}

\paragraph{Critic Loss.} 
We train the critic via a standard Bellman error minimization, such that
\begin{equation}
\label{eq:critic-loss}
   \Lcal_{Q}(\phi) = \left ( Q_\phi (x, a^1) - r - \gamma Q_{\phi}^{\text{target}}(x', a^{\pi}_{x'}) \right)^2
\end{equation} 
where $a^{\pi}_{x'} \sim \pi_{\theta}(\cdot \mid x')$ and $Q_{\phi}^{\text{target}}$ is the target network \citep{mnih2013playing, park2025flow}.

\begin{algorithm}[t]
\DontPrintSemicolon  %
\SetKwComment{tcp}{}{}  %
\SetAlgoNlRelativeSize{-1}  %
\caption{\algo Action Sampling via Forward Euler Method}
\label{alg:sampling}

\KwIn{State $x$, number of inference steps $m$}
\KwOut{Action $a$}

$a \sim \mathcal{N}(0, I)$\;
$h \gets 1/m$\;
$t \gets 0$\;

\For{$n \in \{0, \dots, m - 1\}$}{
  $a \gets a + h \cdot s_{\theta}(a, t, h \mid x)$\;
  $t \gets t + h$\;
}

\KwRet $a$
\end{algorithm}

\subsection{\algo Inference-Time Scaling}
\paragraph{\textcolor{pblue}{The Benefit of Self-Consistency.}}
The loss function $\Lcal_{\pi}(\theta)$ includes both the \emph{flow matching error}
and the \emph{self-consistency error}.
Intuitively, training on the flow matching error alone yields a flow model, from which we can sample actions by numerically solving the flow ODE in Equation \ref{eq:flow-ode} with the Euler method. Using more discretization steps $\dreg$ in the Euler method implies a smaller step size (i.e. $h_{\min}=1/\dreg$ is small), leading to a smaller discretization error. However, the Euler method requires $\Theta(h_{\min}^{-1})$ number of iterations, so a smaller step size $h_{\min}$ is more computationally expensive. By incorporating the self-consistency error into the training objective, the shortcut model approximates two iteration steps of the Euler method into a single step. By iteratively approximating with larger step sizes, the shortcut model achieves discretization error comparable to that of the Euler method with a small step size, while only requiring a constant computational cost.

\paragraph{Sequential Scaling.} 
By training a shortcut model with self-consistency, \algo's policy can perform inference under varying inference budgets. In other words, \algo can sample actions in Algorithm \ref{alg:sampling} with a varying number of inference steps $\dinf$, including one step. 
In order to implement sequential scaling, we simply run the sampling procedure in Algorithm \ref{alg:sampling} for a greater number of steps (i.e. a larger $\dinf$), up to the number of discretization steps $\dreg$ used during training. 

\paragraph{Parallel Scaling.}
We also desire an approach to inference-time scaling that is independent of the number of inference steps. We incorporate best-of-$N$ sampling \citep{lightman2023let,brown2024largelanguagemonkeysscaling}, following the simple procedure: sample actions independently and use a verifier to select the best sample. In \algo, we use the trained Q function as the verifier. We implement best-of-$N$ sampling as follows: given a state $x$, sample $a_1, a_2, \ldots, a_N$ independently from the policy $\pi_\theta(a \mid x)$ and select the action with the largest $Q$ value, such that
\begin{equation}
  \argmax_{a \in \{a_1, a_2, \ldots, a_N\}} Q(x, a)
\end{equation}

\section{Theoretical Analysis: Regularization To Behavior Policy}
Offline RL aims to learn a policy that does not deviate too far from the offline data, in order to avoid test-time distribution shift \citep{levine2020offline}. In this section, we theoretically examine the question:
\begin{center}
    \textit{Will \algo learn a policy that is regularized to the behavior of the offline data?}
\end{center}
Through a novel analysis of shortcut models, we prove that the is \textit{yes}.

In the training objective in Equation~\ref{eq:training-loss}, we include the BC-style flow matching loss $\Lcal_{\text{FM}}$ and self-consistency loss $\Lcal_{\text{SC}}$. The former ensures closeness to the offline data, while the latter allows for fast action generation. In this section, we demonstrate that this training objective
can be interpreted as an instantiation of the constrained policy optimization in Equation~\ref{eq:policy-opt-w-reg}. Unlike the KL-style regularization in Equation~\ref{eq:policy-opt-w-kl-reg}, the $\Lcal_{\text{FM}} + \Lcal_{\text{SC}}$ term in the objective function is a \emph{Wasserstein behavioral regularization}, similar to that of~\cite{park2025flow}. We show that under proper conditions, the shortcut model will generate a distribution close to the target in 2-Wasserstein distance ($\text{W}_2$) in Euclidean norm. 
This implies that as long as we minimize $\Lcal_{\text{FM}} + \Lcal_{\text{SC}}$, we ensure the policies induced by the shortcut model will stay close to the behavior policy in $\text{W}_2$ distance.  

Following the setup of the shortcut model in Section~\ref{sec:background}, we first assume the shortcut model $s(\cdot,\cdot,\cdot)$ is trained properly (i.e. flow-matching and the self-consistency losses are minimized well). In other words, the shortcut model for the smallest step size $s(\cdot,\cdot,\frac{1}{M})$ is close to the ground truth drift function $v_t(\cdot)$, and $s(\cdot,\cdot,2h)$ is consistent with $s(\cdot,\cdot,h)$ on the evaluated time steps in inference.

\begin{table*}[t]
    \vspace{-10pt}
    \caption{
    \textbf{\algo's Overall Performance.} \algo achieves the best performance on 5 of the 8 environments evaluated, for a total of 40 unique tasks. The performance is averaged over 8 seeds, with standard deviations reported. The baseline results are reported from \citet{park2025flow}'s extensive tuning and evaluation of baselines on OGBench tasks. We present the full results in Appendix \ref{appendix:full_results}. 
    }
    \label{table:offline_table_envs}
    \centering
    \vspace{5pt}
    \scalebox{0.58}
    {
    \begin{threeparttable}
    \begin{tabular}{lccccccccccc}
    \toprule
    \multicolumn{1}{c}{} & \multicolumn{3}{c}{\texttt{Gaussian}} & \multicolumn{3}{c}{\texttt{Diffusion}} & \multicolumn{4}{c}{\texttt{Flow}} & \multicolumn{1}{c}{\texttt{Shortcut}} \\
    \cmidrule(lr){2-4} \cmidrule(lr){5-7} \cmidrule(lr){8-11} \cmidrule(lr){12-12}
    \texttt{Task Category} & \texttt{BC} & \texttt{IQL} & \texttt{ReBRAC} & \texttt{IDQL} & \texttt{SRPO} & \texttt{CAC} & \texttt{FAWAC} & \texttt{FBRAC} & \texttt{IFQL} & \texttt{FQL} & \texttt{\color{pblue}\algo} \\
    \midrule
    
    \texttt{OGBench antmaze-large-singletask ($\mathbf{5}$ tasks)} & $11$ {\tiny $\pm 1$} & $53$ {\tiny $\pm 3$} & $81$ {\tiny $\pm 5$} & $21$ {\tiny $\pm 5$} & $11$ {\tiny $\pm 4$} & $33$ {\tiny $\pm 4$} & $6$ {\tiny $\pm 1$} & $60$ {\tiny $\pm 6$} & $28$ {\tiny $\pm 5$} & $79$ {\tiny $\pm 3$} & $\textbf{89}$ {\tiny $\pm 2$} \\
    \texttt{OGBench antmaze-giant-singletask ($\mathbf{5}$ tasks)} & $0$ {\tiny $\pm 0$} & $4$ {\tiny $\pm 1$} & $\mathbf{26}$ {\tiny $\pm 8$} & $0$ {\tiny $\pm 0$} & $0$ {\tiny $\pm 0$} & $0$ {\tiny $\pm 0$} & $0$ {\tiny $\pm 0$} & $4$ {\tiny $\pm 4$} & $3$ {\tiny $\pm 2$} & $9$ {\tiny $\pm 6$} & $9$ {\tiny $\pm 6$} \\
    \texttt{OGBench humanoidmaze-medium-singletask ($\mathbf{5}$ tasks)} & $2$ {\tiny $\pm 1$} & $33$ {\tiny $\pm 2$} & $22$ {\tiny $\pm 8$} & $1$ {\tiny $\pm 0$} & $1$ {\tiny $\pm 1$} & $53$ {\tiny $\pm 8$} & $19$ {\tiny $\pm 1$} & $38$ {\tiny $\pm 5$} & $60$ {\tiny $\pm 14$} & $58$ {\tiny $\pm 5$} & $\textbf{64}$ {\tiny $\pm 4$} \\
    \texttt{OGBench humanoidmaze-large-singletask ($\mathbf{5}$ tasks)} & $1$ {\tiny $\pm 0$} & $2$ {\tiny $\pm 1$} & $2$ {\tiny $\pm 1$} & $1$ {\tiny $\pm 0$} & $0$ {\tiny $\pm 0$} & $0$ {\tiny $\pm 0$} & $0$ {\tiny $\pm 0$} & $2$ {\tiny $\pm 0$} & $\mathbf{11}$ {\tiny $\pm 2$} & $4$ {\tiny $\pm 2$} & $5$ {\tiny $\pm 2$} \\
    \texttt{OGBench antsoccer-arena-singletask ($\mathbf{5}$ tasks)} & $1$ {\tiny $\pm 0$} & $8$ {\tiny $\pm 2$} & $0$ {\tiny $\pm 0$} & $12$ {\tiny $\pm 4$} & $1$ {\tiny $\pm 0$} & $2$ {\tiny $\pm 4$} & $12$ {\tiny $\pm 0$} & $16$ {\tiny $\pm 1$} & $33$ {\tiny $\pm 6$} & $60$ {\tiny $\pm 2$} & $\textbf{69}$ {\tiny $\pm 2$} \\
    \texttt{OGBench cube-single-singletask ($\mathbf{5}$ tasks)} & $5$ {\tiny $\pm 1$} & $83$ {\tiny $\pm 3$} & $91$ {\tiny $\pm 2$} & $\mathbf{95}$ {\tiny $\pm 2$} & $80$ {\tiny $\pm 5$} & $85$ {\tiny $\pm 9$} & $81$ {\tiny $\pm 4$} & $79$ {\tiny $\pm 7$} & $79$ {\tiny $\pm 2$} & $\mathbf{96}$ {\tiny $\pm 1$} & $\textbf{97}$ {\tiny $\pm 1$} \\
    \texttt{OGBench cube-double-singletask ($\mathbf{5}$ tasks)} & $2$ {\tiny $\pm 1$} & $7$ {\tiny $\pm 1$} & $12$ {\tiny $\pm 1$} & $15$ {\tiny $\pm 6$} & $2$ {\tiny $\pm 1$} & $6$ {\tiny $\pm 2$} & $5$ {\tiny $\pm 2$} & $15$ {\tiny $\pm 3$} & $14$ {\tiny $\pm 3$} & $\mathbf{29}$ {\tiny $\pm 2$} & $25$ {\tiny $\pm 3$} \\
    \texttt{OGBench scene-singletask ($\mathbf{5}$ tasks)} & $5$ {\tiny $\pm 1$} & $28$ {\tiny $\pm 1$} & $41$ {\tiny $\pm 3$} & $46$ {\tiny $\pm 3$} & $20$ {\tiny $\pm 1$} & $40$ {\tiny $\pm 7$} & $30$ {\tiny $\pm 3$} & $45$ {\tiny $\pm 5$} & $30$ {\tiny $\pm 3$} & $\mathbf{56}$ {\tiny $\pm 2$} & $\textbf{57}$ {\tiny $\pm 2$} \\

    \bottomrule
    \end{tabular}
    \end{threeparttable}
    }
    \end{table*}

\begin{assum}[Small Flow Matching and Self-Consistency Losses]
\label{assum:fm-cl-err}
There exist $\epsilon_{\text{FM}} > 0$ and $\epsilon_{\text{SC}}>0$, s.t.
\begin{itemize}
    \item for all $t = 0, \frac{1}{M}, \frac{2}{M}, \ldots, 1-\frac{1}{M}$, we have $\EE_{z_t \sim p_t}\left[\|s(z_t,t,\frac{1}{M}) - v_t(z_t)\|_2^2\right] \le \epsilon_{\text{FM}}^2$;
    \item for all $h = \frac{1}{M}, \frac{2}{M}, \frac{2^2}{M}, \ldots, \frac{1}{2}$, and $t = 0,h,2h,\ldots, 1-h$, we have
    \begin{align*}
        \EE_{z_t \sim p_t}\left[\left\|s(z_t, t, h)/2 + s(z'_{t+h}, t, h)/2 - s(z_t,t,2h)\right\|_2^2\right] \le \epsilon_{\text{SC}}^2
    \end{align*}
    where $z'_{t+h} = z_t + s(z_t, t, h)h$.
\end{itemize}    
\end{assum}

\begin{theorem}[Regularization To Behavior Policy] \label{thm:shortcut-conv}
Suppose the shortcut model $s(z,t,h)$ is $L$-Lipschitz in $z$ for all $t$ and $h$, the drift function $v_t(z)$ is $L_v$-Lipschitz in $z$ for all $t$, $\sup_t\EE_{z_t\sim p_t}\left[\left\|v_t\right\|_2^2\right] \le M_v$ and $L/M < 1$. If Assumption~\ref{assum:fm-cl-err} holds, then for all $h = \frac{1}{M}, \frac{2}{M}, \frac{2^2}{M}, \ldots, \frac{1}{2},1$
    \begin{align}
    \text{W}_2(\hat p^{(h)}, p^{\star})
    \le 
    \frac{1}{L}e^{\frac{3}{2}L}\left(
    \underbrace{\frac{eL_v}{M}\left(M_v+1\right)}_{\textnormal{\texttt{Discretization Error}}}
    + 
    \underbrace{\epsilon_{\text{FM}}\vphantom{\frac{eL_v}{M}\left(M_v+1\right)}}_{\textnormal{\texttt{Flow Matching Error}}}
    + 
    \underbrace{\epsilon_{\text{SC}}\log_2 M\vphantom{\frac{eL_v}{M}\left(M_v+1\right)}}_{\textnormal{\texttt{Self-Consistency Error}}}
    \right)
    \end{align}
    where $\hat p^{(h)}$ is the distribution of samples generated by the shortcut model with step size $h$ and $p^\star$ is the data distribution.\footnote{For a cleaner presentation, we consider the unconditional setting and show a uniform upper bound on the Wasserstein distance for all step size $h$. We defer an $h$-dependent upper bound to Appendix~\ref{appendix:proofs}.
    }
\end{theorem}
\paragraph{Discretization Error.} In the upper bound, the term $\frac{eL_v}{M}(M_v + 1)$ corresponds to the \textit{Euler discretization error}---a well-known quantity in numerical ODE solvers---which vanishes as the number of discretization steps $M \to \infty$. This term captures the inherent error from approximating continuous flows with finite-step shortcut models. 

\paragraph{Flow Matching and Self-Consistency Errors.} The terms $ \epsilon_{\text{FM}} $ and $ \epsilon_{\text{SC}} \log_2 M $ account for the training approximation errors: $ \epsilon_{\text{FM}} $ measures the deviation between the shortcut model’s predicted velocity and the true drift under small step sizes, while $ \epsilon_{\text{SC}} \log_2 M $ captures the cumulative consistency error over varying step sizes. Notably, when both $ \Lcal_{\text{FM}} $ and $ \Lcal_{\text{SC}} $ are minimized effectively during training, these errors become negligible. Ignoring the self-consistency loss, our result is comparable to the guarantees for flow models in \citet{roy20242}. However, our analysis incorporates self-consistency and validates it across all discretization levels, which is a novel contribution. 

\paragraph{Regularization To Behavior Policy.} Theorem~\ref{thm:shortcut-conv} shows that when trained properly, the shortcut model generates a distribution close to the target distribution in 2-Wasserstein distance for all step sizes $h$. Thus, the BC-style flow matching loss $\Lcal_{\text{FM}}$ and self-consistency loss $\Lcal_{\text{SC}}$ in the training objective (Equation~\ref{eq:training-loss}) enforce regularization to behavior policy in Wasserstein distance. Consequently, \algo not only learns a performant policy, but also guarantees that the learned policy remains close to the offline data distribution across varying step sizes.

\section{Experiments}
In this section, we evaluate the overall performance of \algo against 10 baselines across 40 tasks. We then investigate \algo's sequential scaling and parallel scaling trends. Furthermore, we present additional results in Appendix \ref{appendix:full_results} and ablation studies in Appendix \ref{appendix:ablations}.

\subsection{Experimental Setup}
\paragraph{Environments and Tasks.} We evaluate \algo on locomotion and manipulation robotics tasks in the OGBench task suite \citep{park2024ogbench}. The experimental setup in this section follows the setup suggested by \citet{park2024ogbench, park2025flow}. We document the complete implementation details in Appendix \ref{appen:impl_det}. Following \citet{park2025flow}, we use OGBench's reward-based \texttt{singletask} variants for all experiments, which are best suited for reward-maximizing RL.  

\paragraph{Baselines.} We evaluate against three Gaussian-based offline RL algorithms (\texttt{BC} \citep{pomerleau1988alvinn}, \texttt{IQL} \citep{kostrikov2021offline}, \texttt{ReBRAC} \citep{tarasov2023revisiting}), three diffusion-based algorithms (\texttt{IDQL} \citep{hansen2023idql}, \texttt{SRPO} \citep{chen2023score}, \texttt{CAC} \citep{ding2023consistency}), and four flow-based algorithms (\texttt{FAWAC} \citep{nair2020awac,park2025flow}, \texttt{FBRAC} \citep{zhang2025energy,park2025flow}, \texttt{IFQL} \citep{wang2022diffusion,park2025flow}, \texttt{FQL} \citep{park2025flow}). For the baselines we compare against in this paper, we report results from \citet{park2025flow}, who performed an extensive tuning and evaluation of the aforementioned baselines on OGBench tasks. We provide a thorough discussion of the baselines in Appendix \ref{appen:impl_det}. 

\begin{figure}[t!]
  \centering
  \includegraphics[width=0.30\textwidth]{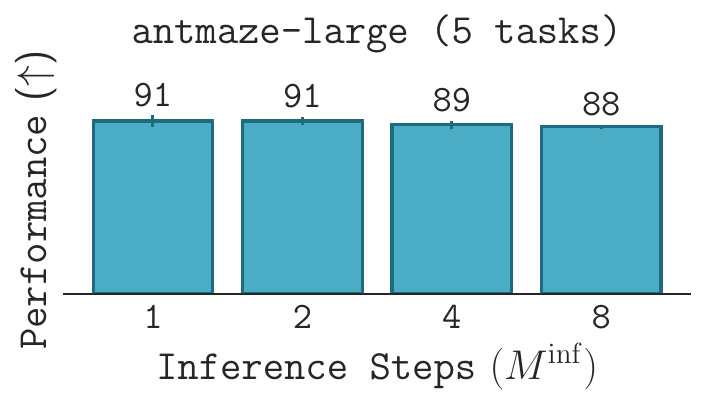}
  \includegraphics[width=0.30\textwidth]{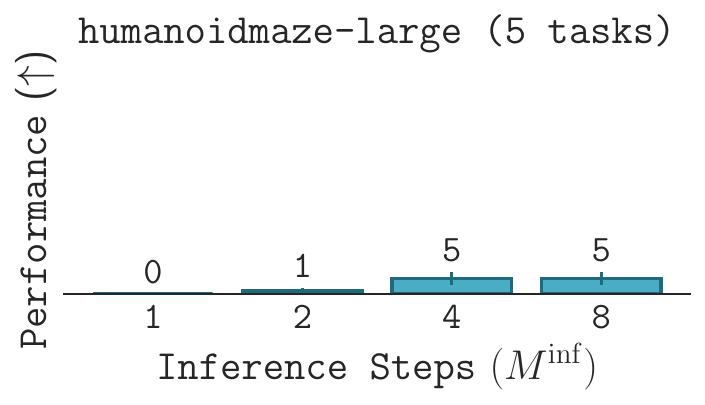} \\
  \includegraphics[width=0.30\textwidth]{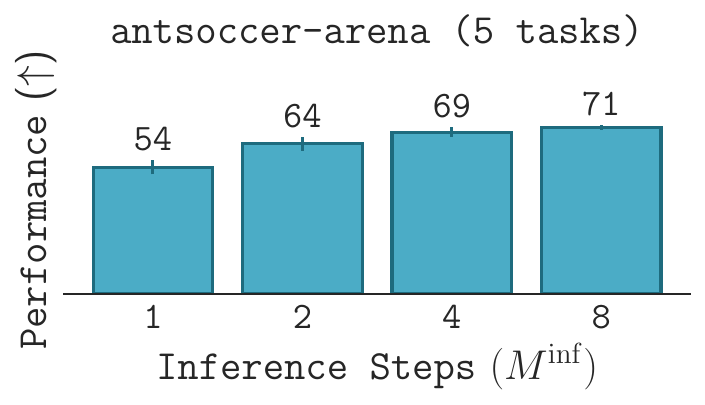}
  \includegraphics[width=0.30\textwidth]{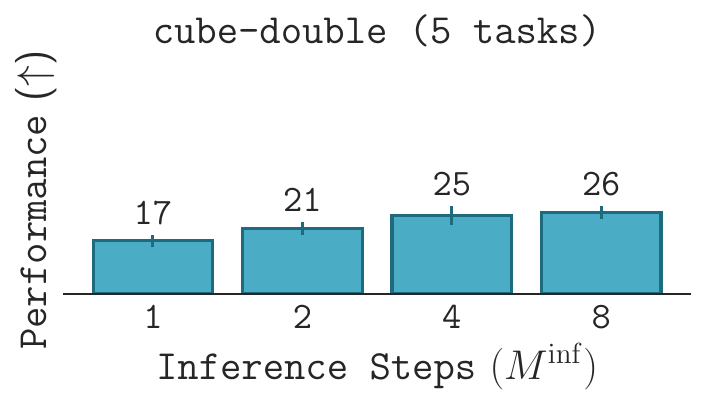}
  \includegraphics[width=0.30\textwidth]{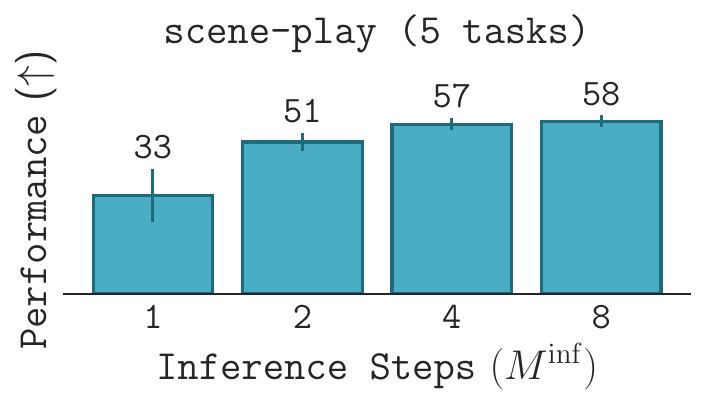}
  \caption{\textbf{\algo's Sequential Scaling.} For a fixed training budget, \algo generally improves performance with greater test-time compute. We fix a training budget of discretization steps and backpropagation steps through time ($\dreg = \dbtt = 8$) and vary the inference budget via the number of inference steps $\dinf$. The performance is averaged over 8 seeds for each task, with 5 tasks per environment, and standard deviations reported.
  }
  \label{fig:sec-scal}
\end{figure}

\paragraph{Evaluation.}  Following \citet{park2025flow}, we ensure a fair comparison by using the same network size, number of gradient steps, and discount factor for all algorithms. Furthermore, we hyperparameter tune \algo with a similar training budget to the baselines: we only tune one of \algo's training parameters on the \textit{default} task in each environment. We average over 8 seeds per task and report standard deviations in tables. We bold values at 95\% of the best performance in tables. For \algo, we use 8 discretization steps during training (2 fewer than the baselines), since \algo requires that discretization steps be powers of 2. We detail all other parameters in Appendix~\ref{appen:impl_det}.

\begin{figure}
  \centering
  \centering
  \includegraphics[width=0.60\textwidth]{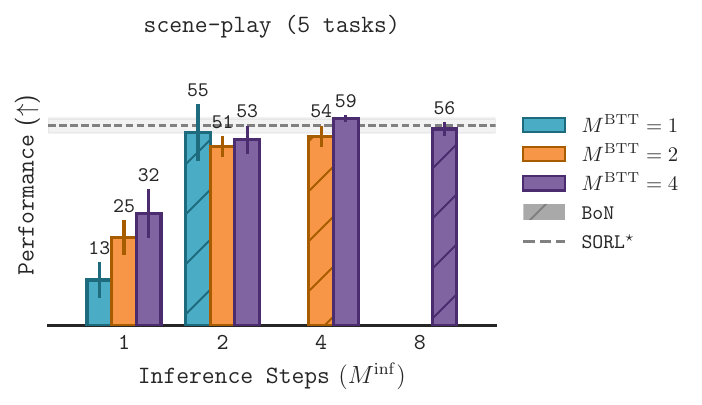}
  \caption{\textbf{\algo's Parallel Scaling.} \algo generalizes to \textit{new} inference steps at test-time, beyond what was optimized through backpropagation during training. For each fixed training budget (i.e. fixed number of discretization steps $\dreg$ and backpropagation through time steps $\dbtt$), we evaluate with varying inference steps $\dinf$. $\dbtt$ denotes the maximum number of steps used for backpropagation through time in the $Q$ update. The $\star$ hatch denotes best-of-$N$ sampling, with $N=8$, where the number of inference steps is \textit{greater} than the number of backpropagation steps through time (i.e. $\dinf > \dbtt$). $\algo^{\star}$ denotes the best performance achieved by \algo in Table \ref{table:offline_table_envs}. Results are averaged over 8 seeds for each of the 5 tasks. 
  }
  \label{fig:par-scal}
\end{figure}

\subsection{Experimental Results}
\label{sec:results:over-perf}

\textbf{Q: Does \algo perform well on different environments?}

\textit{Yes, \algo achieves the best performance on the majority of the diverse set of 40 tasks.}

We present \algo's overall performance across a range of environments in Table \ref{table:offline_table_envs}. Notably, \algo achieves the best performance on 5 out of 8 environments, including substantial improvements over the baselines on \texttt{antmaze-large} and \texttt{antsoccer-arena}. The results suggest that while distillation approaches like \texttt{FQL} \citep{park2025flow} can achieve high performance, some tasks require greater expressiveness or precision than can be achieved from one-step policies.

\textbf{Q: For a fixed training budget, can \algo improve performance purely by test-time scaling?} 

\textit{Yes, \algo's performance is improved by increasing the number of inference steps at test-time.}

We investigate \algo's sequential scaling by plotting the results of varying inference steps $\dinf$ given the same, fixed training budget (i.e. holding the discretization steps $\dreg$ and the steps backpropagated through time $\dbtt$ constant), for $\dinf \leq \dbtt$. We plot the results in Figure~\ref{fig:sec-scal}. 
The results show noticeable improvement in performance as the number of inference steps increases, suggesting that \algo scales positively with greater test-time compute.

\textbf{Q: Can \algo make up for less training-time compute with greater test-time compute?}

\textit{Yes, given less training compute, \algo can match optimal performance with greater test-time compute.}

We reduce the training budget by reducing the steps of backpropagation through time from $\dbtt=8$ for Table \ref{table:offline_table_envs} and Figure \ref{fig:sec-scal} to $\dbtt=1, 2, 4$ for Figure \ref{fig:par-scal}. However, through a combination of inference
sequential and parallel scaling, we recover the performance of the optimal policy from the best training budget ($\algo^\star$).

\textbf{Q: At test-time, can \algo generalize to inference steps beyond what it was trained on?}

\textit{Yes, through sequential and parallel scaling, \algo can use a greater number of inference steps at test-time than the number of backpropagation steps used during training.}

Given a fixed training budget (i.e. a fixed number of discretization steps $\dreg$ and backpropagation steps through time $\dbtt$), we evaluate on an increasing number of inference steps $\dinf$, coupled with best-of-$N$ sampling. The bar colors in Figure \ref{fig:par-scal} denote different training compute, dictated by the number of backpropagation through time steps $\dbtt$. 
We apply best-of-$N$ sampling on a greater number of inference steps than were used for backpropagation through time (i.e. $\dinf > \dbtt$), thus testing \algo's ability to generalize to inference steps beyond what it was backpropagated on. From Figure \ref{fig:par-scal}, \algo \textit{can} use more inference steps than the number of backpropagation steps used during training (i.e. the number of steps used for backpropagation through time $\dbtt$), up to a performance saturation point (approximately $\dbtt=4$).

Although Figure~\ref{fig:par-scal} shows that best-of-$N$ sampling can improve performance, the gain is \textit{not} theoretically guaranteed. Our verifier is the \emph{learned} value estimator, not a ground-truth reward, so there is no statistical benefit to learning a verifier \citep{swamy2025all}.\footnote{In other words, if the actor were already acting greedily with respect to $Q_\phi$ (i.e. the case of perfect optimization), ranking additional samples by $Q_\phi$ cannot raise the expected value.} 
Empirically, however, best-of-$N$ effectively functions as an additional, inference-time \emph{policy-extraction} step, searching over nearby actions and selecting the one with the highest $Q_\phi$---similar to \citet{park2024value}'s work. Thus, the post-hoc filtering may still be empirically beneficial in regaining performance lost to imperfect optimization.

\section{Related Work}
\label{sec:related-work}
\paragraph{Offline Reinforcement Learning.} The goal of offline reinforcement learning (RL)~\citep{levine2020offline} is to learn a policy solely through previously collected data, without further interaction with the environment. A considerable amount of prior work has been developed, with the core idea being to maximize returns while minimizing a discrepancy measure between the state-action distribution of the dataset and that of the learned policy. This goal has been pursued through various strategies: behavioral regularization~\citep{nair2020awac,fujimoto2021minimalist,tarasov2023revisiting}, conservative value estimation~\citep{kumar2020conservative}, in-distribution maximization~\citep{kostrikov2021offline,xu2023offline,garg2023extreme}, out-of-distribution detection~\citep{yu2020mopo,kidambi2020morel,an2021uncertainty,nikulin2023anti}, dual formulations of RL~\citep{lee2021optidice,sikchi2023dual}, and representing policy models using generative modeling~\cite{chen2021decision,janner2021offline,janner2022planning,park2025flow}. After training an offline RL policy, it can be further fine-tuned with additional online rollouts, which is referred to as offline-to-online RL, for which several techniques have been proposed~\citep{lee2021optidice,song2022hybrid,nakamoto2023cal,ball2023efficient,yu2023actor,ren2024hybrid}. 

\paragraph{Reinforcement Learning with Generative Models.} Motivated by the recent success of iterative generative modeling techniques, such as denoising diffusion~\citep{sohl2015deep,ho2020denoising,song2021scorebased} and flow matching~\citep{lipman2024flow,esser2024scaling}, the use of generative models as a policy class for imitation learning and reinforcement learning has shown promise due to its expressiveness for multimodal action distributions~\citep{wang2022diffusion,ren2024diffusion,wu2024diffusing,black2024pi0}. However, its iterative noise sampling process leads to a large time consumption and memory occupancy~\citep{ding2023consistency}. Some methods utilize a two-stage approach: first training an iterative sampling model before then distilling it~\cite{frans2024one,ho2020denoising,meng2023distillation}, but this may lead to performance degradation and introduce complexity from the distillation models, along with error compounding across the two-stage procedure. Consistency models~\citep{song2023improved} are another type of unified model, but they rely on extensive bootstrapping, such as requiring a specific learning schedule during training \citep{frans2024one}, making them difficult to train. 

\paragraph{Inference-Time Scaling.}
Recent advances in large language models (LLMs)~\citep{openai2024openaio1card, deepseekai2025deepseekr1incentivizingreasoningcapability} have demonstrated the ability to increase performance at inference time with more compute through parallel scaling methods~\citep{wang2023selfconsistencyimproveschainthought, gui2024bonbon, brown2024largelanguagemonkeysscaling,pan2025learning}, sequential scaling methods that extend the depth of reasoning by increasing the chain of thought budget~\citep{wei2023chainofthoughtpromptingelicitsreasoning, muennighoff2025s1simpletesttimescaling,qin2025backtrack}, and self-correcting methods~\citep{madaan2023selfrefineiterativerefinementselffeedback, qu2024recursiveintrospectionteachinglanguage, kumar2024traininglanguagemodelsselfcorrect}.

Generative models such as diffusion and flow models inherently support inference-time sequential scaling by varying the number of denoising steps~\citep{sohl2015deep, lipman2024flow}, and recent work extends their inference-time scaling capabilities through parallel scaling methods via reward-guided sampling~\citep{ma2025inferencetimescalingdiffusionmodels, singhal2025generalframeworkinferencetimescaling}. However, prior reinforcement learning methods that use generative policies lose this inherent inference-time sequential scaling since they require the same number of denoising steps at both training and inference \citep{wang2022diffusion, kang2023efficientdiffusionpoliciesoffline, zhang2024entropy}. In contrast, \algo supports both sequential and parallel scaling at inference time, allowing for dynamic trade-offs between compute and performance, and improved action selection via the learned $Q$ function. 

Independent of generative models, prior work has proposed applying a form of best-of-$N$ sampling to actions from the \textit{behavior} policy (i.e. \textit{behavioral} candidates) \citep{chen2022offline,fujimoto2019off,ghasemipour2021emaq,hansen2023idql,park2024value}. \citet{park2024value} proposed two methods of test-time policy improvement, by using the gradient of the $Q$-function. One of \citet{park2024value}'s approaches relies on leveraging test-time states, which \algo's parallel scaling method does not require. The second approach proposed by \citet{park2024value} adjusts actions using the gradient of the learned $Q$-function, which is conceptually similar to our approach of best-of-$N$ sampling with the $Q$-function verifier. However, their method requires an additional hyperparameter to tune the update magnitude in gradient space.

\section{Discussion}
\label{sec:discussion}

We introduce \algo: a simple, efficient one-stage training procedure for expressive policies in offline RL. \algo's key property is \textit{self-consistency}, which enables expressive inference under \textit{any} inference budget, including one-step. Theoretically, we prove that \algo regularizes to the behavior policy, using a novel analysis of shortcut models. Empirically, \algo demonstrates the best performance across a range of diverse tasks. Additionally, \algo can be scaled at test-time, empirically demonstrating how greater compute at inference-time can further improve performance. 
An avenue for future work is to investigate how \algo can incorporate adaptive test-time scaling \citep{pan2025learning,ma2025inferencetimescalingdiffusionmodels}, such as selecting the number of inference steps based on the gradient of the $Q$-function.

\newpage
\bibliography{bibliography}
\bibliographystyle{plainnat}

\newpage
\appendix
\section{Limitations}
\label{appendix:limitations}

As noted in Section \ref{sec:results:over-perf}, the positive trend in parallel scaling may not occur across all environments. Our approach to parallel scaling uses the Q function as a verifier. If the learned Q function is inaccurate or out of distribution, then additional optimization may not be beneficial \citep{levine2020offline}. Additionally, while \algo is highly \textit{flexible}, allowing for scaling of both training-time and inference-time compute budgets, \algo generally has a longer training runtime than \fql, one of the fastest flow-based baselines \citep{park2025flow}, as evidenced by the runtime comparison (Figure~\ref{fig:runtime}) in Appendix \ref{appendix:full_results}. We believe that trading off greater training runtime for improved performance is desirable in offline RL, since training does not require interaction with an expert, environment, or simulator. Finally, as noted in \citet{park2025flow}, \textit{offline} RL algorithms, including \algo, lack a principled exploration strategy that may be necessary for attaining optimal performance in \textit{online} RL.

\newpage
\section{Proofs}
\label{appendix:proofs}
We first present the upper bound on the Wasserstein distance with explicit dependency on $h$.
\begin{theorem}[Restatement of Theorem~\ref{thm:shortcut-conv} With Explicit Dependency on $h$] \label{thm:shortcut-conv-explicit-bound}
Suppose the shortcut model $s(z,t,h)$ is $L$-Lipschitz in $z$ for all $t$ and $h$, the drift function $v_t(z)$ is $L_v$-Lipschitz in $z$ for all $t$, $\sup_t\EE_{z_t\sim p_t}\left[\left\|v_t\right\|_2^2\right] \le M_v$ and $L/M < 1$. If Assumption~\ref{assum:fm-cl-err} holds, then for all $h = \frac{1}{M}, \frac{2}{M}, \frac{2^2}{M}, \ldots, \frac{1}{2},1$
    \begin{align*}
        \text{W}_2(\hat p^{(h)}, p^{\star})
        \le & \frac{1}{L}\left((1+Lh)^{\frac{1}{h}}-1\right)\exp\left(\frac{1}{2}L h\right)\left(\frac{eL_v}{M} \left(M_v+1\right) + \epsilon_{FM} + \epsilon_{\text{SC}} \log_2 (Mh)\right)
    \end{align*}
    where $\hat p^{(h)}$ is the distribution of samples generated by the shortcut model with step size $h$ and $p^\star$ is the data distribution.
\end{theorem}

\subsection{Proof of Theorem \ref{thm:shortcut-conv} / \ref{thm:shortcut-conv-explicit-bound}}

This theorem formalizes the key claim: minimizing the training objective keeps the learned policy close to the behavior policy. This closeness is measured using a strong metric---Wasserstein distance---implying the model will not ``drift'' far from the offline data distribution during generation. Under Lipschitz assumptions and small flow-matching/self-consistency loss (Assumption \ref{assum:fm-cl-err}), the shortcut model generates samples whose distribution is close to the target data distribution in 2-Wasserstein distance ($W_2$), uniformly over all discretization step sizes $h$. 

\paragraph{Proof Sketch.}
At a high level, our goal is to show that, starting from the same $z_0 \sim \Ncal(0,I)$, running the sampling process of the shortcut model and running the ground truth flow ODE yield similar output measured by square error averaged over $z_0 \sim \Ncal(0,I)$.  

This argument constructs a coupling with marginal distributions $\hat p^{(h)}$ and $p^{\star}$ and small transport cost. By definition, the Wasserstein distance between $\hat p^{(h)}$ and $p^{\star}$ is also small.

Our proof follows from three steps, described below. 

\paragraph{(1) Small Step Error.} We first show that the shortcut model provides a good local approximation to the true dynamics when trained well at the smallest step size. In other words, under small flow-matching error and Lipschitz drift dynamics, the shortcut model has bounded error when running a single inference step with the smallest stepsize:
\begin{lemma}[Single-Step Error with Minimum Step-Size]\label{lem:h0-single-step-error}
  Let $h_0 = \frac{1}{M}$ be the smallest stepsize. If
  \begin{itemize}
  \item for all $t = 0, h_0, 2h_0, \ldots, 1-h_0$, $E_{z_t \sim P_t}\left[\|s(z_t,t,h_0) - v_t(z_t))\|_2^2\right] \le \epsilon_{\text{FM}}^2$,
  \item $v_t(x)$ is $L$-Lipschitz in both $t$ and $x$;
  \item for all $t = 0, h_0, 2h_0, \ldots, 1-h_0$, $E_{z_t \sim P_t}\left[\|v_t(z_t))\|_2^2\right] \le M_v^2$,
  \end{itemize}
  then
  \begin{equation}
    \EE_{z_t \sim P_t}\left[\|z_t + s(z_t, t, h_0)h_0 - F(z_t, t, t + h_0)\|_2^2\right] \le h_0^2\left(L_ve^{L_vh_0}h_0\left(M_v+1\right) + \epsilon_{\text{FM}}\right)^2.
  \end{equation}
\end{lemma}
\paragraph{(2) Larger Step Error.} We further demonstrate that the single inference step remains accurate even at coarser step sizes, thanks to the self-consistency constraint:
\begin{lemma}[Single-Step Error with Step Size $h$]\label{lem:h-single-step-error}
  Let $h_0 = \frac{1}{M}$ be the smallest stepsize. If
  \begin{itemize}
  \item for all $h'$ and $t = 0,h', 2h', \ldots, 1-h'$, $\EE\left[\|F^{(2h')}(z_t,t,t+2h') - F^{(h')}(z_t,t,t+2h')\|_2^2\right] \le 4h'^2\epsilon_{\text{SC}}^2$;
  \item for all $t = 0, h_0, 2h_0, \ldots, 1-h_0$, $\EE_{z_t \sim P_t}\left[\|z_t + s(z_t, t, h_0)h_0 - F(z_t, t, t + h_0)\|_2^2\right] \le h_0^2\epsilon^2$;
  \item for all $h'$ and $t = 0,h', 2h', \ldots, 1-h'$, $s(\cdot, t, h')$ is $L$-Lipschitz,
  \end{itemize}
  then for $h = 2^n h_0$, and for all $t = 0, h, \ldots, 1-h$,
  \begin{equation}
    \sqrt{\EE\left[\|F^{(h')}(z_t,t,t+h') - F(z_t,t,t+h')\|_2^2\right]}
    \le
    h \exp\left(\frac{1}{2}L h\right)\left(\epsilon + n\epsilon_{\text{SC}}\right).
  \end{equation}
\end{lemma}

\paragraph{(3) Composed Error Over Multiple Steps.} Finally, even with repeated applications, the shortcut model’s errors remain controlled, showing stability over long horizons. In other words, if the single inference step error is small and $s$ is $L$-Lipschitz, then the multi-step trajectory error is bounded:
\begin{lemma}[Error of $1/h$-Step Inference]\label{lem:multi-step-error}
  If $\EE_{z_t \sim P_t}\left[\|z_t + s(z_t, t, h)h - F(z_t, t, t + h)\|_2^2\right] \le h^2\epsilon^2$ for all $t = 0, h, 2h, \ldots, 1$, and $s(\cdot, t, h)$ is $L$-Lipschitz, then
  \begin{equation}
    \sqrt{\EE\left[\|\hat z_1^{(h)} - z_1\|_2^2\right]} \le \left((1+Lh)^{\frac{1}{h}}-1\right)\frac{\epsilon}{L}.
  \end{equation}
\end{lemma}

\paragraph{Summary: Aggregation of Error Bounds.} Given these insights, we finish the proof of the main theorem by applying the lemmas in order.

\begin{proof}[Proof of Theorem \ref{thm:shortcut-conv} / \ref{thm:shortcut-conv-explicit-bound}]
Let $h_0 = \frac{1}{M}$ be the smallest step size. Suppose $h = 2^nh_0$.

By Lemma~\ref{lem:h0-single-step-error},
  \begin{equation}
    \EE_{z_t \sim P_t}\left[\|z_t + s(z_t, t, h_0)h_0 - F(z_t, t, t + h_0)\|_2^2\right] \le h_0^2\left(L_ve^{L_vh_0}h_0\left(M_v+1\right) + \epsilon_{\text{FM}}\right)^2.
  \end{equation}
By Lemma~\ref{lem:h-single-step-error},
\begin{equation}
    \sqrt{\EE\left[\|F^{(h)}(z_t,t,t+h) - F(z_t,t,t+h)\|_2^2\right]}
    \le
    h \exp\left(\frac{1}{2}L h\right)\left(L_ve^{L_vh_0}h_0\left(M_v+1\right) + \epsilon_{\text{FM}} + n\epsilon_{\text{SC}}\right).
  \end{equation}
By Lemma~\ref{lem:multi-step-error},
\begin{align}
    \sqrt{\EE\left[\|\hat z_1^{(h)} - z_1\|_2^2\right]} \le \frac{1}{L}\left((1+Lh)^{\frac{1}{h}}-1\right)\exp\left(\frac{1}{2}L h\right)\left(L_ve^{L_vh_0}h_0\left(M_v+1\right) + \epsilon_{\text{FM}} + n\epsilon_{\text{SC}}\right)
  \end{align}
\end{proof}

\newpage
\subsection{Proof of Lemma \ref{lem:h0-single-step-error}}

\paragraph{Proof Sketch.}  The proof's strategy is:
\begin{enumerate}
    \item Compare the true dynamics $z_t$ with linear approximation $\bar z_t$ using ODE analysis.
    \item Bound the deviation over time.
    \item Add the flow-matching loss to account for model error.
\end{enumerate}

\begin{proof}
  Consider $t \in [t', t'+h_0]$.
  Let
  \begin{align}
    z_{t} :=  & F(z_{t'}, t', t) \\
    \bar z_{t} := & z_{t'} + v_{t'}(z_{t'}) (t-t'). 
  \end{align}
  By definition,
  \begin{align}
    \frac{\d z_t}{\d t} = & v_t(z_t)\\
    \frac{\d \bar z_t}{\d t} = & v_{t'}(z_{t'})
  \end{align}
  Then
  \begin{align}\label{eq:mse-ode-1}
    \frac{\d}{\d t}\|z_t - \bar z_t\|_2^2 
    &=  2\left\langle z_t - \bar z_t, \frac{\d}{\d t}z_t - \frac{\d}{\d t}\bar z_t\right\rangle \\
    &=  2\left\langle z_t - \bar z_t, v_t(z_t) - v_{t'}(z_{t'}) \right\rangle \\
    &\le  2\|z_t - \bar z_t\|_2 \|v_t(z_t) - v_{t'}(z_{t'})\|_2 
  \end{align}
  On the other hand, by chain rule:
  \begin{equation}
    \label{eq:mse-ode-2}
    \frac{\d}{\d t}\|z_t - \bar z_t\|_2^2
    = 2\|z_t - \bar z_t\|_2\frac{\d}{\d t}\|z_t - \bar z_t\|_2
  \end{equation}
  By Equations~\ref{eq:mse-ode-1} and~\ref{eq:mse-ode-2}:
  \begin{align}
    \frac{\d}{\d t}\|z_t - \bar z_t\|_2 \le &  \|v_t(z_t) - v_{t'}(z_{t'})\|_2
  \end{align}
  By triangle inequality:
  \begin{align}
    \|v_t(z_t) - v_{t'}(z_{t'})\|_2
    \le & \|v_t(z_t) - v_{t}(\bar z_{t})\|_2 + \|v_{t}(\bar z_{t}) - v_{t}(z_{t'})\|_2 + \|v_{t}(z_{t'}) - v_{t'}(z_{t'})\|_2
  \end{align}  
  Because $v_t(x)$ is $L_v$-Lipschitz in $x$,
  \begin{equation}
    \|v_t(z_t) - v_{t}(\bar z_{t})\|_2 \le L_v \|z_t - \bar z_t\|_2.
  \end{equation}
  Because $\bar z_t = z_{t'} + v_{t'}(z_{t'})(t - t')$ and $v_t(x)$ is $L_v$-Lipschitz in $x$,
  \begin{align}
    \|v_{t}(\bar z_{t}) - v_{t}(z_{t'})\|_2 
    &=  \|v_{t}(z_{t'} + v_{t'}(z_{t'})(t - t')) - v_{t}(z_{t'})\|_2 \\
    &\le L_v\|v_{t'}(z_{t'})\|_2(t - t') \\
    &\le L_v \|v_{t'}(z_{t'})\|_2h_0.
  \end{align}
  Because $v_t(x)$ is $L_v$-Lipschitz in $t$,
  \begin{align}
    \|v_{t}(z_{t'}) - v_{t'}(z_{t'})\|_2 \le L_v(t-t')\le L_v h_0.
  \end{align}
  Thus
  \begin{align}
    \frac{\d}{\d t}\|z_t - \bar z_t\|_2 \le &  L_v \|z_t - \bar z_t\|_2 + L_v(\|v_{t'}(z_{t'})\|_2+1)h_0
  \end{align}  
  Multiplying $e^{-L_v(t-t')}$ on both side, we get:
  \begin{align}
    e^{-L_v(t-t')} \frac{\d}{\d t}\|z_t - \bar z_t\|_2 \le e^{-L_v(t-t')} L_v\|z_t - \bar z_t\|_2  +  e^{-L_v(t-t')}L_v(\|v_{t'}(z_{t'})\|_2+1)h_0
  \end{align}
  By rearranging:
  \begin{align}
    \frac{\d}{\d t}\left(e^{-L_v(t-t')} \|z_t - \bar z_t\|_2\right) 
    &\le e^{-L_v(t-t')}L_v(\|v_{t'}(z_{t'})\|_2+1)h_0 \\
    &\le L_v(\|v_{t'}(z_{t'})\|_2+1)h_0.
  \end{align}
  Because $z_{t'} = \bar z_{t'}$, by integrating both side over $t \in [t', t'+h_0]$ , we have:
  \begin{align}
    e^{-L_vh_0} \|z_{t' + h_0} - \bar z_{t'+h_0}\|_2 \le L_v(\|v_{t'}(z_{t'})\|_2+1)h_0^2.
  \end{align}
  Thus
  \begin{align}
    \|z_{t' + h_0} - \bar z_{t'+h_0}\|_2 \le &  L_v(\|v_{t'}(z_{t'})\|_2+1)e^{L_vh_0}h_0^2.
  \end{align}
  Taking square and expectation on both sides, we have:
  \begin{align}
    \EE\left[\|z_{t' + h_0} - \bar z_{t'+h_0}\|_2^2\right] \le &  L_v^2e^{2L_vh_0}h_0^4\EE\left[(\|v_{t'}(z_{t'})\|_2+1)^2\right].
  \end{align}
  Thus
  \begin{align}
    \sqrt{\EE\left[\|z_{t' + h_0} - \bar z_{t'+h_0}\|_2^2\right]} 
    & \le  L_ve^{L_vh_0}h_0^2\sqrt{\EE\left[(\|v_{t'}(z_{t'})\|_2+1)^2\right]} \\
    & \le L_ve^{L_vh_0}h_0^2\left(\sqrt{\EE\left[\|v_{t'}(z_{t'})\|_2^2\right]}+1\right)\\
    & \le L_ve^{L_vh_0}h_0^2\left(M_v+1\right)
  \end{align}  
  By definition,
  \begin{align}
    \|\bar z_{t' + h_0} - \hat z_{t'+h_0}\|_2 
    & = \left\|z_{t'} + v_{t'}(z_{t'}) h_0 - (z_{t'} + s(z_{t'},t',h_0)h_0)\right\|_2 \\
    & = \left\|v_{t'}(z_{t'}) - s(z_{t'},t',h_0)\right\|_2h_0
  \end{align}
  Taking square and expectation on both sides, we have
  \begin{align}
    \EE\left[\|\bar z_{t' + h_0} - \hat z_{t'+h_0}\|_2^2\right] = & h_0^2\EE\left[\left\|v_{t'}(z_{t'}) - s(z_{t'},t',h_0)\right\|_2^2\right]
                                                                    \le h_0^2 \epsilon_{\text{FM}}^2.
  \end{align}  
  Thus
  \begin{align}
    \sqrt{\EE_{z_{t'}}\left[\|z_{t'+h_0}-\hat z_{t'+h_0}\|_2^2\right]} 
    & \le \sqrt{\EE_{z_{t'}}\left[\|z_{t'+h_0}-\bar z_{t'+h_0}\|_2^2\right]} + \sqrt{\EE_{z_{t'}}\left[\|\bar z_{t'+h_0}-\hat z_{t'+h_0}\|_2^2\right]}\\
    & \le Le^{L_vh_0}h_0^2\left(M_v+1\right) + h_0 \epsilon_{\text{FM}} \\
    & \le h_0\left(Le^{L_vh_0}h_0\left(M_v+1\right) + \epsilon_{\text{FM}}\right)
  \end{align}
\end{proof}

\newpage 
\subsection{Proof of Lemma \ref{lem:h-single-step-error}}

\paragraph{Proof Sketch.} The proof's strategy is:
\begin{enumerate}
    \item Define a recursive error relationship for step sizes $2^k h_0$.
    \item Use inductive bounding and logarithmic scaling of error growth.
    \item Solve the recursion to show controlled error amplification.
\end{enumerate}

\begin{proof}
  We define $\Delta_{h'}$ to be the maximum 1-step error induced by shortcut model with step size $h'$. Formally, for all $h'>0$ s.t. $1/h' \in \ZZ$, we define, 
  \begin{equation}
    \label{eq:1-step-sc-err}
    \Delta_{h'} := \max_{t \in \{0,h'\ldots,1-h'\}} \sqrt{\EE\left[\|F^{(h')}(z_t,t,t+h') - F(z_t,t,t+h')\|_2^2\right]}
  \end{equation}
  For all $h'$ and $t \in \{0,2h'\ldots,1-2h'\}$,
    \begin{align}
      &\sqrt{\EE\left[\left\|F^{(2h')}(z_t,t,t+2h') - F(z_t,t,t+2h')\right\|_2^2\right]} \nonumber\\
      \le & \sqrt{\EE\left[\left\|F^{(2h')}(z_t,t,t+2h') - F^{(h')}(z_t,t,t+2h')\right\|_2^2\right]} \nonumber\\
      & + \sqrt{\EE\left[\left\|F^{(h')}(z_t,t,t+2h') - F(z_t,t,t+2h')\right\|_2^2\right]}
    \end{align}
  By assumption, the first term is bounded by $2h'\epsilon_{\text{SC}}$.
  We now analyze the second term.
  Let $\hat z_{t+h'} := F^{(h')}(z_t,t,t+h')$, then
  \begin{align}
    & \sqrt{\EE\left[\|F^{(h')}(z_t,t,t+2h') - F(z_t,t,t+2h')\|_2^2\right]} \\
    = & \sqrt{\EE\left[\|F^{(h')}(\hat z_{t+h'},t+h',t+2h') - F(z_{t+h'},t+h',t+2h')\|_2^2\right]} \\
    \le & \sqrt{\EE\left[\|F^{(h')}(\hat z_{t+h'},t+h',t+2h') - F^{(h')}(z_{t+h'},t+h',t+2h')\|_2^2\right]} \\
    & + \sqrt{\EE\left[\|F^{(h')}(z_{t+h'},t+h',t+2h') - F(z_{t+h'},t+h',t+2h')\|_2^2\right]}.
  \end{align}
  By triangle inequality,
  \begin{align}
    & \|F^{(h')}(\hat z_{t+h'},t+h',t+2h') - F^{(h')}(z_{t+h'},t+h',t+2h')\|_2 \\
    = & \left\|(\hat z_{t+h'} + s(\hat z_{t+h'},t+h',t+2h')h') - (z_{t+h'} + s(z_{t+h'},t+h',t+2h')h')\right\|_2\\
    \le & \left\|\hat z_{t+h'} - z_{t+h'}\right\|_2 + h'\left\|s(\hat z_{t+h'},t+h',h') - s(z_{t+h'},t+h',h')\right\|_2
  \end{align}
  Because $s(\cdot,t,h')$ is $L$-Lipschitz (Assumption~\ref{assum:fm-cl-err}),
  \begin{align}
    & \|F^{(h')}(\hat z_{t+h'},t+h',t+2h') - F^{(h')}(z_{t+h'},t+h',t+2h')\|_2
      \le (1+Lh')\|\hat z_{t+h'} - z_{t+h'}\|_2
  \end{align}
  Thus
  \begin{align}
    & \sqrt{\EE\left[\|F^{(h')}(\hat z_{t+h'},t+h',t+2h') - F^{(h')}(z_{t+h'},t+h',t+2h')\|_2^2\right]} \\
    \le & (1+L h')\sqrt{\EE\left[\|\hat z_{t+h'} - z_{t+h'}\|_2^2\right]} \\
    = & (1+L h')\sqrt{\EE\left[\|F^{(h')}(z_t,t,t+h') - F(z_t,t,t+h')\|_2^2\right]} \le (1+Lh')\Delta_{h'}.
  \end{align}
  By Equation~\ref{eq:1-step-sc-err},
  \begin{align}
    \sqrt{\EE\left[\|F^{(h')}(z_{t+h'},t+h',t+2h') - F(z_{t+h'},t+h',t+2h')\|_2^2\right]} \le \Delta_{h'}.
  \end{align}
  Combine everything together,
  \begin{align}
    & \sqrt{\EE\left[\|F^{(2h')}(z_t,t,t+2h') - F(z_t,t,t+2h')\|_2^2\right]} \\
    \le & 2h'\epsilon_{\text{SC}} + (1+Lh')\Delta_{h'} + \Delta_{h'} = 2h'\epsilon_{\text{SC}} + (2+Lh')\Delta_{h'}.
  \end{align}
  Because $t'$ is chosen arbitrarily, we have
  \begin{align}
    \Delta_{2h'}
    \le & 2h'\epsilon_{\text{SC}} + (2+Lh')\Delta_{h'}.
  \end{align}
  Let $h' = 2^k h_0$ and $A_k := \Delta_{2^k h_0}$, we have:
  \begin{align}
    A_{k+1} \le (2+L h_0 2^k) A_k + 2\epsilon_{\text{SC}} h_0 2^{k}.
  \end{align}
  Solving this recursion, we have:
  \begin{align}
    A_n \le & A_0\prod_{j=0}^{n-1}(2+Lh_02^j) +  2\epsilon_{\text{SC}}h_0\sum_{k=0}^{n-1}2^{k}\prod_{j=k+1}^{n-1}(2+Lh_02^j)\\
    = & A_0 2^{n} \prod_{j=0}^{n-1}\left(1+\frac{1}{2}Lh_02^j\right) +  2\epsilon_{\text{SC}}h_0\sum_{k=0}^{n-1}2^{k}2^{n-k-1}\prod_{j=k+1}^{n-1}\left(1+\frac{1}{2}Lh_02^j\right)\\
    = & A_0 2^{n} \prod_{j=0}^{n-1}\left(1+\frac{1}{2}Lh_02^j\right) +  2^n\epsilon_{\text{SC}}h_0\sum_{k=0}^{n-1}\prod_{j=k+1}^{n-1}\left(1+\frac{1}{2}Lh_02^j\right).
  \end{align}
  Because $\ln(1+x) \le x$ for all $x> -1$, we have
  \begin{align}
    A_n \le 
    & A_0 2^{n} \exp\left(\sum_{j=0}^{n-1}\ln\left(1+\frac{1}{2}Lh_02^j\right)\right) \\
    & + 2^n\epsilon_{\text{SC}}h_0\sum_{k=0}^{n-1}\exp\left(\sum_{j=k+1}^{n-1}\ln\left(1+\frac{1}{2}Lh_02^j\right)\right) \\
    \le & A_0 2^{n} \exp\left(\frac{1}{2}Lh_0\sum_{j=0}^{n-1}2^j\right) +  2^n\epsilon_{\text{SC}}h_0\sum_{k=0}^{n-1}\exp\left(\frac{1}{2}Lh_0\sum_{j=k+1}^{n-1}2^j\right) \\
    = & A_0 2^{n} \exp\left(\frac{1}{2}Lh_0(2^n-1)\right) +  2^n\epsilon_{\text{SC}}h_0\sum_{k=0}^{n-1}\exp\left(\frac{1}{2}Lh_02^{k+1}(2^{n-k-1}-1)\right) \\
    \le & A_0 2^{n} \exp\left(\frac{1}{2}Lh_02^n\right) +  2^n\epsilon_{\text{SC}}h_0\sum_{k=0}^{n-1}\exp\left(\frac{1}{2}Lh_02^{n}\right) \\
    = & 2^n h_0 \exp\left(\frac{1}{2}Lh_02^n\right)\left(\frac{A_0}{h_0} + n\epsilon_{\text{SC}}\right) \\
    = & h \exp\left(\frac{1}{2}L h\right)\left(\frac{A_0}{h_0} + n\epsilon_{\text{SC}}\right)  
  \end{align}
  Because $\EE_{z_t \sim P_t}\left[\|z_t + s(z_t, t, h_0)h_0 - F(z_t, t, t + h_0)\|_2^2\right] \le h_0^2\epsilon^2$, we have $A_0 \le h_0\epsilon$. Thus
  \begin{align}
    \Delta_h = \Delta_{2^n h_0} = A_n \le h \exp\left(\frac{1}{2}L h\right)\left(\epsilon + n\epsilon_{\text{SC}}\right).
    \end{align}
\end{proof}

\newpage 
\subsection{Proof of Lemma \ref{lem:multi-step-error}}

\paragraph{Proof Sketch.}  The proof's strategy is:
\begin{enumerate}
    \item Define the recurrence over time steps for error accumulation.
    \item Use Lipschitz continuity to control error propagation.
    \item Solve the recurrence analytically.
\end{enumerate}

\begin{proof}
  \begin{align}
    \sqrt{\EE\left[\|\hat z_{t+h}^{(h)} - z_{t+h}\|_2^2\right]}
    \le & \sqrt{\EE\left[\|\hat z_t^{(h)} + s(\hat z_t^{(h)}, t, h)h - (z_t + s(z_t, t, h)h)\|_2^2\right]} \nonumber \\
    &+ \sqrt{\EE\left[\|(z_t + s(z_t, t, h)h) - z_{t+h}\|_2^2\right]} \\
    \le & (1+Lh) \sqrt{\EE\left[\|\hat z_{t}^{(h)} - z_{t}\|_2^2\right]} + h\epsilon
  \end{align}
  Because $\EE[\|\hat z_0^{(h)} - z_0\|_2^2] =0$, by solving the recursion,
  \begin{align}
    \sqrt{\EE\left[\|\hat z_1^{(h)} - z_1\|_2^2\right]} \le \left((1+Lh)^{\frac{1}{h}}-1\right)\frac{\epsilon}{L}
  \end{align}
\end{proof}

\newpage
\section{Full Results}
\label{appendix:full_results}

\vspace{3em}
\begin{table}[H]
    \vspace{-30pt}
    \caption{
      \textbf{\algo's Overall Performance By Task.}
      We present the full results on the OGBench tasks. \texttt{(*)} indicates the default task in each environment.
      The results are averaged over $8$ seeds with standard deviations reported. The baseline results are reported from \citet{park2025flow}'s extensive tuning and evaluation of baselines on OGBench tasks.
    }
    \label{table:offline_table_tasks}
    \centering
    \vspace{5pt}
    \resizebox{\textwidth}{!}{
    \begin{threeparttable}
    \begin{tabular}{lccccccccccc}
    \toprule
    \multicolumn{1}{c}{} & \multicolumn{3}{c}{\texttt{Gaussian}} & \multicolumn{3}{c}{\texttt{Diffusion}} & \multicolumn{4}{c}{\texttt{Flow}} & \multicolumn{1}{c}{\texttt{Shortcut}} \\
    \cmidrule(lr){2-4} \cmidrule(lr){5-7} \cmidrule(lr){8-11} \cmidrule(lr){12-12}
    \texttt{Task} & \texttt{BC} & \texttt{IQL} & \texttt{ReBRAC} & \texttt{IDQL} & \texttt{SRPO} & \texttt{CAC} & \texttt{FAWAC} & \texttt{FBRAC} & \texttt{IFQL} & \texttt{FQL} & \texttt{\color{pblue}\algo} \\
    \midrule
    
    \texttt{antmaze-large-navigate-singletask-task1-v0 (*)} & $0$ {\tiny $\pm 0$} & $48$ {\tiny $\pm 9$} & $\mathbf{91}$ {\tiny $\pm 10$} & $0$ {\tiny $\pm 0$} & $0$ {\tiny $\pm 0$} & $42$ {\tiny $\pm 7$} & $1$ {\tiny $\pm 1$} & $70$ {\tiny $\pm 20$} & $24$ {\tiny $\pm 17$} & $80$ {\tiny $\pm 8$} & $\textbf{93}$ {\tiny $\pm 2$} \\
    \texttt{antmaze-large-navigate-singletask-task2-v0} & $6$ {\tiny $\pm 3$} & $42$ {\tiny $\pm 6$} & $\mathbf{88}$ {\tiny $\pm 4$} & $14$ {\tiny $\pm 8$} & $4$ {\tiny $\pm 4$} & $1$ {\tiny $\pm 1$} & $0$ {\tiny $\pm 1$} & $35$ {\tiny $\pm 12$} & $8$ {\tiny $\pm 3$} & $57$ {\tiny $\pm 10$} & $79$ {\tiny $\pm 5$} \\
    \texttt{antmaze-large-navigate-singletask-task3-v0} & $29$ {\tiny $\pm 5$} & $72$ {\tiny $\pm 7$} & $51$ {\tiny $\pm 18$} & $26$ {\tiny $\pm 8$} & $3$ {\tiny $\pm 2$} & $49$ {\tiny $\pm 10$} & $12$ {\tiny $\pm 4$} & $83$ {\tiny $\pm 15$} & $52$ {\tiny $\pm 17$} & $\mathbf{93}$ {\tiny $\pm 3$} & $\mathbf{88}$ {\tiny $\pm 10$} \\
    \texttt{antmaze-large-navigate-singletask-task4-v0} & $8$ {\tiny $\pm 3$} & $51$ {\tiny $\pm 9$} & $84$ {\tiny $\pm 7$} & $62$ {\tiny $\pm 25$} & $45$ {\tiny $\pm 19$} & $17$ {\tiny $\pm 6$} & $10$ {\tiny $\pm 3$} & $37$ {\tiny $\pm 18$} & $18$ {\tiny $\pm 8$} & $80$ {\tiny $\pm 4$} & $\mathbf{91}$ {\tiny $\pm 2$} \\
    \texttt{antmaze-large-navigate-singletask-task5-v0} & $10$ {\tiny $\pm 3$} & $54$ {\tiny $\pm 22$} & $\mathbf{90}$ {\tiny $\pm 2$} & $2$ {\tiny $\pm 2$} & $1$ {\tiny $\pm 1$} & $55$ {\tiny $\pm 6$} & $9$ {\tiny $\pm 5$} & $76$ {\tiny $\pm 8$} & $38$ {\tiny $\pm 18$} & $83$ {\tiny $\pm 4$} & $\mathbf{95}$ {\tiny $\pm 0$} \\
    \midrule
    \texttt{antmaze-giant-navigate-singletask-task1-v0 (*)} & $0$ {\tiny $\pm 0$} & $0$ {\tiny $\pm 0$} & $\mathbf{27}$ {\tiny $\pm 22$} & $0$ {\tiny $\pm 0$} & $0$ {\tiny $\pm 0$} & $0$ {\tiny $\pm 0$} & $0$ {\tiny $\pm 0$} & $0$ {\tiny $\pm 1$} & $0$ {\tiny $\pm 0$} & $4$ {\tiny $\pm 5$} & $12$ {\tiny $\pm 6$} \\
    \texttt{antmaze-giant-navigate-singletask-task2-v0} & $0$ {\tiny $\pm 0$} & $1$ {\tiny $\pm 1$} & $\mathbf{16}$ {\tiny $\pm 17$} & $0$ {\tiny $\pm 0$} & $0$ {\tiny $\pm 0$} & $0$ {\tiny $\pm 0$} & $0$ {\tiny $\pm 0$} & $4$ {\tiny $\pm 7$} & $0$ {\tiny $\pm 0$} & $9$ {\tiny $\pm 7$} & $0$ {\tiny $\pm 0$} \\
    \texttt{antmaze-giant-navigate-singletask-task3-v0} & $0$ {\tiny $\pm 0$} & $0$ {\tiny $\pm 0$} & $\mathbf{34}$ {\tiny $\pm 22$} & $0$ {\tiny $\pm 0$} & $0$ {\tiny $\pm 0$} & $0$ {\tiny $\pm 0$} & $0$ {\tiny $\pm 0$} & $0$ {\tiny $\pm 0$} & $0$ {\tiny $\pm 0$} & $0$ {\tiny $\pm 1$} & $0$ {\tiny $\pm 0$} \\
    \texttt{antmaze-giant-navigate-singletask-task4-v0} & $0$ {\tiny $\pm 0$} & $0$ {\tiny $\pm 0$} & $5$ {\tiny $\pm 12$} & $0$ {\tiny $\pm 0$} & $0$ {\tiny $\pm 0$} & $0$ {\tiny $\pm 0$} & $0$ {\tiny $\pm 0$} & $9$ {\tiny $\pm 4$} & $0$ {\tiny $\pm 0$} & $14$ {\tiny $\pm 23$} & $\mathbf{25}$ {\tiny $\pm 18$} \\
    \texttt{antmaze-giant-navigate-singletask-task5-v0} & $1$ {\tiny $\pm 1$} & $19$ {\tiny $\pm 7$} & $\mathbf{49}$ {\tiny $\pm 22$} & $0$ {\tiny $\pm 1$} & $0$ {\tiny $\pm 0$} & $0$ {\tiny $\pm 0$} & $0$ {\tiny $\pm 0$} & $6$ {\tiny $\pm 10$} & $13$ {\tiny $\pm 9$} & $16$ {\tiny $\pm 28$} & $6$ {\tiny $\pm 15$} \\
    \midrule
    \texttt{humanoidmaze-medium-navigate-singletask-task1-v0 (*)} & $1$ {\tiny $\pm 0$} & $32$ {\tiny $\pm 7$} & $16$ {\tiny $\pm 9$} & $1$ {\tiny $\pm 1$} & $0$ {\tiny $\pm 0$} & $38$ {\tiny $\pm 19$} & $6$ {\tiny $\pm 2$} & $25$ {\tiny $\pm 8$} & $\mathbf{69}$ {\tiny $\pm 19$} & $19$ {\tiny $\pm 12$} & $\mathbf{67}$ {\tiny $\pm 4$} \\
    \texttt{humanoidmaze-medium-navigate-singletask-task2-v0} & $1$ {\tiny $\pm 0$} & $41$ {\tiny $\pm 9$} & $18$ {\tiny $\pm 16$} & $1$ {\tiny $\pm 1$} & $1$ {\tiny $\pm 1$} & $47$ {\tiny $\pm 35$} & $40$ {\tiny $\pm 2$} & $76$ {\tiny $\pm 10$} & $85$ {\tiny $\pm 11$} & $\mathbf{94}$ {\tiny $\pm 3$} & $\mathbf{89}$ {\tiny $\pm 3$} \\
    \texttt{humanoidmaze-medium-navigate-singletask-task3-v0} & $6$ {\tiny $\pm 2$} & $25$ {\tiny $\pm 5$} & $36$ {\tiny $\pm 13$} & $0$ {\tiny $\pm 1$} & $2$ {\tiny $\pm 1$} & $\mathbf{83}$ {\tiny $\pm 18$} & $19$ {\tiny $\pm 2$} & $27$ {\tiny $\pm 11$} & $49$ {\tiny $\pm 49$} & $74$ {\tiny $\pm 18$} & $\mathbf{83}$ {\tiny $\pm 4$} \\
    \texttt{humanoidmaze-medium-navigate-singletask-task4-v0} & $0$ {\tiny $\pm 0$} & $0$ {\tiny $\pm 1$} & $\mathbf{15}$ {\tiny $\pm 16$} & $1$ {\tiny $\pm 1$} & $1$ {\tiny $\pm 1$} & $5$ {\tiny $\pm 4$} & $1$ {\tiny $\pm 1$} & $1$ {\tiny $\pm 2$} & $1$ {\tiny $\pm 1$} & $3$ {\tiny $\pm 4$} & $1$ {\tiny $\pm 0$} \\
    \texttt{humanoidmaze-medium-navigate-singletask-task5-v0} & $2$ {\tiny $\pm 1$} & $66$ {\tiny $\pm 4$} & $24$ {\tiny $\pm 20$} & $1$ {\tiny $\pm 1$} & $3$ {\tiny $\pm 3$} & $91$ {\tiny $\pm 5$} & $31$ {\tiny $\pm 7$} & $63$ {\tiny $\pm 9$} & $\mathbf{98}$ {\tiny $\pm 2$} & $\mathbf{97}$ {\tiny $\pm 2$} & $81$ {\tiny $\pm 20$} \\
    \midrule
    \texttt{humanoidmaze-large-navigate-singletask-task1-v0 (*)} & $0$ {\tiny $\pm 0$} & $3$ {\tiny $\pm 1$} & $2$ {\tiny $\pm 1$} & $0$ {\tiny $\pm 0$} & $0$ {\tiny $\pm 0$} & $1$ {\tiny $\pm 1$} & $0$ {\tiny $\pm 0$} & $0$ {\tiny $\pm 1$} & $6$ {\tiny $\pm 2$} & $7$ {\tiny $\pm 6$} & $\mathbf{20}$ {\tiny $\pm 9$} \\
    \texttt{humanoidmaze-large-navigate-singletask-task2-v0} & $\mathbf{0}$ {\tiny $\pm 0$} & $\mathbf{0}$ {\tiny $\pm 0$} & $\mathbf{0}$ {\tiny $\pm 0$} & $\mathbf{0}$ {\tiny $\pm 0$} & $\mathbf{0}$ {\tiny $\pm 0$} & $\mathbf{0}$ {\tiny $\pm 0$} & $\mathbf{0}$ {\tiny $\pm 0$} & $\mathbf{0}$ {\tiny $\pm 0$} & $\mathbf{0}$ {\tiny $\pm 0$} & $\mathbf{0}$ {\tiny $\pm 0$} & $\mathbf{0}$ {\tiny $\pm 0$} \\
    \texttt{humanoidmaze-large-navigate-singletask-task3-v0} & $1$ {\tiny $\pm 1$} & $7$ {\tiny $\pm 3$} & $8$ {\tiny $\pm 4$} & $3$ {\tiny $\pm 1$} & $1$ {\tiny $\pm 1$} & $2$ {\tiny $\pm 3$} & $1$ {\tiny $\pm 1$} & $10$ {\tiny $\pm 2$} & $\mathbf{48}$ {\tiny $\pm 10$} & $11$ {\tiny $\pm 7$} & $5$ {\tiny $\pm 2$} \\
    \texttt{humanoidmaze-large-navigate-singletask-task4-v0} & $1$ {\tiny $\pm 0$} & $1$ {\tiny $\pm 0$} & $1$ {\tiny $\pm 1$} & $0$ {\tiny $\pm 0$} & $0$ {\tiny $\pm 0$} & $0$ {\tiny $\pm 1$} & $0$ {\tiny $\pm 0$} & $0$ {\tiny $\pm 0$} & $1$ {\tiny $\pm 1$} & $\mathbf{2}$ {\tiny $\pm 3$} & $0$ {\tiny $\pm 0$} \\
    \texttt{humanoidmaze-large-navigate-singletask-task5-v0} & $0$ {\tiny $\pm 1$} & $1$ {\tiny $\pm 1$} & $\mathbf{2}$ {\tiny $\pm 2$} & $0$ {\tiny $\pm 0$} & $0$ {\tiny $\pm 0$} & $0$ {\tiny $\pm 0$} & $0$ {\tiny $\pm 0$} & $1$ {\tiny $\pm 1$} & $0$ {\tiny $\pm 0$} & $1$ {\tiny $\pm 3$} & $0$ {\tiny $\pm 0$} \\
    \midrule
    \texttt{antsoccer-arena-navigate-singletask-task1-v0} & $2$ {\tiny $\pm 1$} & $14$ {\tiny $\pm 5$} & $0$ {\tiny $\pm 0$} & $44$ {\tiny $\pm 12$} & $2$ {\tiny $\pm 1$} & $1$ {\tiny $\pm 3$} & $22$ {\tiny $\pm 2$} & $17$ {\tiny $\pm 3$} & $61$ {\tiny $\pm 25$} & $77$ {\tiny $\pm 4$} & $\mathbf{93}$ {\tiny $\pm 4$} \\
    \texttt{antsoccer-arena-navigate-singletask-task2-v0} & $2$ {\tiny $\pm 2$} & $17$ {\tiny $\pm 7$} & $0$ {\tiny $\pm 1$} & $15$ {\tiny $\pm 12$} & $3$ {\tiny $\pm 1$} & $0$ {\tiny $\pm 0$} & $8$ {\tiny $\pm 1$} & $8$ {\tiny $\pm 2$} & $75$ {\tiny $\pm 3$} & $88$ {\tiny $\pm 3$} & $\mathbf{96}$ {\tiny $\pm 2$} \\
    \texttt{antsoccer-arena-navigate-singletask-task3-v0} & $0$ {\tiny $\pm 0$} & $6$ {\tiny $\pm 4$} & $0$ {\tiny $\pm 0$} & $0$ {\tiny $\pm 0$} & $0$ {\tiny $\pm 0$} & $8$ {\tiny $\pm 19$} & $11$ {\tiny $\pm 5$} & $16$ {\tiny $\pm 3$} & $14$ {\tiny $\pm 22$} & $\mathbf{61}$ {\tiny $\pm 6$} & $55$ {\tiny $\pm 6$} \\
    \texttt{antsoccer-arena-navigate-singletask-task4-v0 (*)} & $1$ {\tiny $\pm 0$} & $3$ {\tiny $\pm 2$} & $0$ {\tiny $\pm 0$} & $0$ {\tiny $\pm 1$} & $0$ {\tiny $\pm 0$} & $0$ {\tiny $\pm 0$} & $12$ {\tiny $\pm 3$} & $24$ {\tiny $\pm 4$} & $16$ {\tiny $\pm 9$} & $39$ {\tiny $\pm 6$} & $\mathbf{54}$ {\tiny $\pm 5$} \\
    \texttt{antsoccer-arena-navigate-singletask-task5-v0} & $0$ {\tiny $\pm 0$} & $2$ {\tiny $\pm 2$} & $0$ {\tiny $\pm 0$} & $0$ {\tiny $\pm 0$} & $0$ {\tiny $\pm 0$} & $0$ {\tiny $\pm 0$} & $9$ {\tiny $\pm 2$} & $15$ {\tiny $\pm 4$} & $0$ {\tiny $\pm 1$} & $36$ {\tiny $\pm 9$} & $\mathbf{47}$ {\tiny $\pm 9$} \\
    \midrule
    \texttt{cube-single-play-singletask-task1-v0} & $10$ {\tiny $\pm 5$} & $88$ {\tiny $\pm 3$} & $89$ {\tiny $\pm 5$} & $\mathbf{95}$ {\tiny $\pm 2$} & $89$ {\tiny $\pm 7$} & $77$ {\tiny $\pm 28$} & $81$ {\tiny $\pm 9$} & $73$ {\tiny $\pm 33$} & $79$ {\tiny $\pm 4$} & $\mathbf{97}$ {\tiny $\pm 2$} & $\mathbf{97}$ {\tiny $\pm 2$} \\
    \texttt{cube-single-play-singletask-task2-v0 (*)} & $3$ {\tiny $\pm 1$} & $85$ {\tiny $\pm 8$} & $92$ {\tiny $\pm 4$} & $\mathbf{96}$ {\tiny $\pm 2$} & $82$ {\tiny $\pm 16$} & $80$ {\tiny $\pm 30$} & $81$ {\tiny $\pm 9$} & $83$ {\tiny $\pm 13$} & $73$ {\tiny $\pm 3$} & $\mathbf{97}$ {\tiny $\pm 2$} & $\mathbf{99}$ {\tiny $\pm 0$} \\
    \texttt{cube-single-play-singletask-task3-v0} & $9$ {\tiny $\pm 3$} & $91$ {\tiny $\pm 5$} & $93$ {\tiny $\pm 3$} & $\mathbf{99}$ {\tiny $\pm 1$} & $\mathbf{96}$ {\tiny $\pm 2$} & $\mathbf{98}$ {\tiny $\pm 1$} & $87$ {\tiny $\pm 4$} & $82$ {\tiny $\pm 12$} & $88$ {\tiny $\pm 4$} & $\mathbf{98}$ {\tiny $\pm 2$} & $\mathbf{99}$ {\tiny $\pm 1$} \\
    \texttt{cube-single-play-singletask-task4-v0} & $2$ {\tiny $\pm 1$} & $73$ {\tiny $\pm 6$} & $\mathbf{92}$ {\tiny $\pm 3$} & $\mathbf{93}$ {\tiny $\pm 4$} & $70$ {\tiny $\pm 18$} & $\mathbf{91}$ {\tiny $\pm 2$} & $79$ {\tiny $\pm 6$} & $79$ {\tiny $\pm 20$} & $79$ {\tiny $\pm 6$} & $\mathbf{94}$ {\tiny $\pm 3$} & $\mathbf{95}$ {\tiny $\pm 2$} \\
    \texttt{cube-single-play-singletask-task5-v0} & $3$ {\tiny $\pm 3$} & $78$ {\tiny $\pm 9$} & $87$ {\tiny $\pm 8$} & $\mathbf{90}$ {\tiny $\pm 6$} & $61$ {\tiny $\pm 12$} & $80$ {\tiny $\pm 20$} & $78$ {\tiny $\pm 10$} & $76$ {\tiny $\pm 33$} & $77$ {\tiny $\pm 7$} & $\mathbf{93}$ {\tiny $\pm 3$} & $\mathbf{93}$ {\tiny $\pm 3$} \\
    \midrule
    \texttt{cube-double-play-singletask-task1-v0} & $8$ {\tiny $\pm 3$} & $27$ {\tiny $\pm 5$} & $45$ {\tiny $\pm 6$} & $39$ {\tiny $\pm 19$} & $7$ {\tiny $\pm 6$} & $21$ {\tiny $\pm 8$} & $21$ {\tiny $\pm 7$} & $47$ {\tiny $\pm 11$} & $35$ {\tiny $\pm 9$} & $61$ {\tiny $\pm 9$} & $\mathbf{77}$ {\tiny $\pm 11$} \\
    \texttt{cube-double-play-singletask-task2-v0 (*)} & $0$ {\tiny $\pm 0$} & $1$ {\tiny $\pm 1$} & $7$ {\tiny $\pm 3$} & $16$ {\tiny $\pm 10$} & $0$ {\tiny $\pm 0$} & $2$ {\tiny $\pm 2$} & $2$ {\tiny $\pm 1$} & $22$ {\tiny $\pm 12$} & $9$ {\tiny $\pm 5$} & $\mathbf{36}$ {\tiny $\pm 6$} & $33$ {\tiny $\pm 8$} \\
    \texttt{cube-double-play-singletask-task3-v0} & $0$ {\tiny $\pm 0$} & $0$ {\tiny $\pm 0$} & $4$ {\tiny $\pm 1$} & $17$ {\tiny $\pm 8$} & $0$ {\tiny $\pm 1$} & $3$ {\tiny $\pm 1$} & $1$ {\tiny $\pm 1$} & $4$ {\tiny $\pm 2$} & $8$ {\tiny $\pm 5$} & $\mathbf{22}$ {\tiny $\pm 5$} & $12$ {\tiny $\pm 6$} \\
    \texttt{cube-double-play-singletask-task4-v0} & $0$ {\tiny $\pm 0$} & $0$ {\tiny $\pm 0$} & $1$ {\tiny $\pm 1$} & $0$ {\tiny $\pm 1$} & $0$ {\tiny $\pm 0$} & $0$ {\tiny $\pm 1$} & $0$ {\tiny $\pm 0$} & $0$ {\tiny $\pm 1$} & $1$ {\tiny $\pm 1$} & $5$ {\tiny $\pm 2$} & $\mathbf{7}$ {\tiny $\pm 4$} \\
    \texttt{cube-double-play-singletask-task5-v0} & $0$ {\tiny $\pm 0$} & $4$ {\tiny $\pm 3$} & $4$ {\tiny $\pm 2$} & $1$ {\tiny $\pm 1$} & $0$ {\tiny $\pm 0$} & $3$ {\tiny $\pm 2$} & $2$ {\tiny $\pm 1$} & $2$ {\tiny $\pm 2$} & $17$ {\tiny $\pm 6$} & $\mathbf{19}$ {\tiny $\pm 10$} & $1$ {\tiny $\pm 1$} \\
    \midrule
    \texttt{scene-play-singletask-task1-v0} & $19$ {\tiny $\pm 6$} & $94$ {\tiny $\pm 3$} & $\mathbf{95}$ {\tiny $\pm 2$} & $\mathbf{100}$ {\tiny $\pm 0$} & $94$ {\tiny $\pm 4$} & $\mathbf{100}$ {\tiny $\pm 1$} & $87$ {\tiny $\pm 8$} & $\mathbf{96}$ {\tiny $\pm 8$} & $\mathbf{98}$ {\tiny $\pm 3$} & $\mathbf{100}$ {\tiny $\pm 0$} & $\mathbf{99}$ {\tiny $\pm 1$} \\
    \texttt{scene-play-singletask-task2-v0 (*)} & $1$ {\tiny $\pm 1$} & $12$ {\tiny $\pm 3$} & $50$ {\tiny $\pm 13$} & $33$ {\tiny $\pm 14$} & $2$ {\tiny $\pm 2$} & $50$ {\tiny $\pm 40$} & $18$ {\tiny $\pm 8$} & $46$ {\tiny $\pm 10$} & $0$ {\tiny $\pm 0$} & $76$ {\tiny $\pm 9$} & $\mathbf{89}$ {\tiny $\pm 9$} \\
    \texttt{scene-play-singletask-task3-v0} & $1$ {\tiny $\pm 1$} & $32$ {\tiny $\pm 7$} & $55$ {\tiny $\pm 16$} & $\mathbf{94}$ {\tiny $\pm 4$} & $4$ {\tiny $\pm 4$} & $49$ {\tiny $\pm 16$} & $38$ {\tiny $\pm 9$} & $78$ {\tiny $\pm 14$} & $54$ {\tiny $\pm 19$} & $\mathbf{98}$ {\tiny $\pm 1$} & $\mathbf{97}$ {\tiny $\pm 1$} \\
    \texttt{scene-play-singletask-task4-v0} & $2$ {\tiny $\pm 2$} & $0$ {\tiny $\pm 1$} & $3$ {\tiny $\pm 3$} & $4$ {\tiny $\pm 3$} & $0$ {\tiny $\pm 0$} & $0$ {\tiny $\pm 0$} & $\mathbf{6}$ {\tiny $\pm 1$} & $4$ {\tiny $\pm 4$} & $0$ {\tiny $\pm 0$} & $5$ {\tiny $\pm 1$} & $1$ {\tiny $\pm 1$} \\
    \texttt{scene-play-singletask-task5-v0} & $\mathbf{0}$ {\tiny $\pm 0$} & $\mathbf{0}$ {\tiny $\pm 0$} & $\mathbf{0}$ {\tiny $\pm 0$} & $\mathbf{0}$ {\tiny $\pm 0$} & $\mathbf{0}$ {\tiny $\pm 0$} & $\mathbf{0}$ {\tiny $\pm 0$} & $\mathbf{0}$ {\tiny $\pm 0$} & $\mathbf{0}$ {\tiny $\pm 0$} & $\mathbf{0}$ {\tiny $\pm 0$} & $\mathbf{0}$ {\tiny $\pm 0$} & $\mathbf{0}$ {\tiny $\pm 0$} \\
    
    \bottomrule
    \end{tabular}
    \end{threeparttable}
    }
    \vspace{-10pt}
    \end{table}

\vspace{1em}

\subsection{Per-Task Results}

\textbf{Q: What is \algo's overall performance on each task?}

\textit{\algo achieves the best performance on the majority of the diverse set of tasks considered.}

We present \algo's overall performance for each individual task in Table~\ref{table:offline_table_tasks}. Note that the only \algo training parameter that changes between environments is the \texttt{Q-loss coefficient}, which is common in offline RL \citep{tarasov2023corl,park2024value,park2025flow}. The \texttt{Q-loss coefficient} was tuned on the default task in each environment, and the same coefficient was used on all 5 tasks in each environment. The experimental setup was kept consistent to ensure a fair comparison between \algo and the baselines reported in \citet{park2025flow}. 

\clearpage

\subsection{Runtime Comparison}
\begin{figure}[t!]
  \centering
  \includegraphics[width=0.495\textwidth]{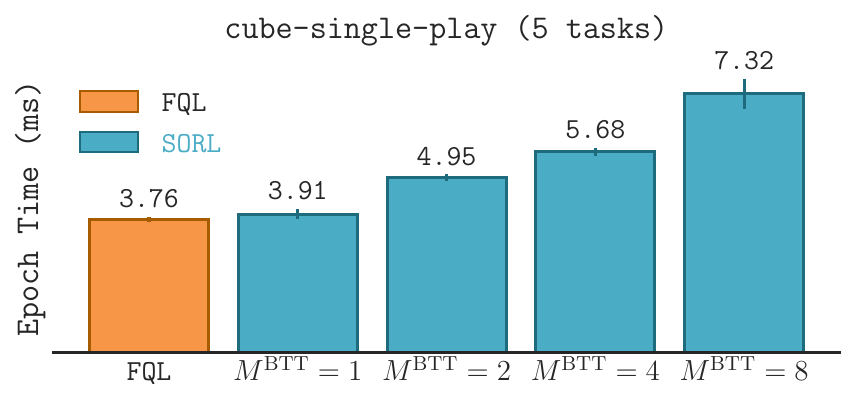}
  \includegraphics[width=0.495\textwidth]{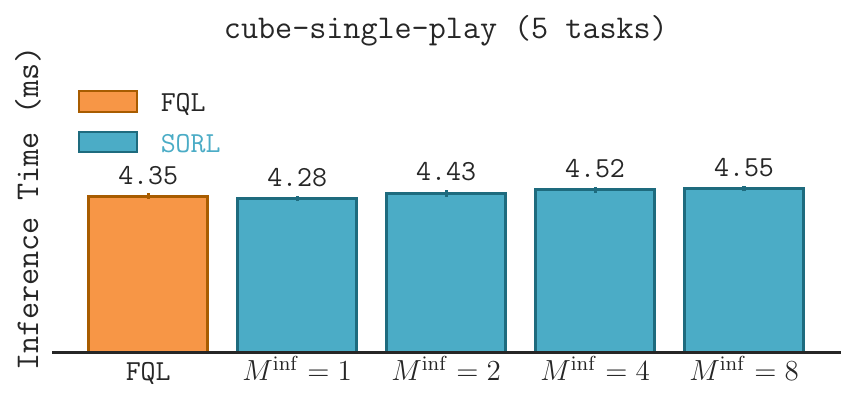}
  \caption{\textbf{Runtime Comparison.} We vary \algo's \textit{training-time compute budget} (i.e. the number of backpropagation steps through time $\dbtt$) on the \textit{left} and \algo's \textit{inference-time compute budget} (i.e. the number of inference steps $\dinf$) on the \textit{right}. The performance is averaged over 5 seeds for each task, with 5 tasks per environment, and standard deviations reported.
  }
  \label{fig:runtime}
\end{figure}

\textbf{Q: How does \algo's runtime compare to the fastest flow-based baseline, \fql?}

\textit{\algo's scalability enables training runtime to match that of \fql under low compute budgets, while exceeding it when larger budgets are allocated. During inference, \algo maintains runtime comparable to \fql.}

We present a runtime comparison between \algo and \texttt{FQL}. \texttt{FQL} is one of the fastest flow-based baselines we consider, for both training and inference \citep{park2025flow}, as \fql's policy is a one-step distillation model. We select \texttt{cube-single-play} because \algo and \fql achieve similar performance on the environment (Table~\ref{table:offline_table_envs}). The experiments were performed on a Nvidia RTX 3090 GPU. The epoch and inference times are averaged over the first three evaluations (i.e. the first training/evaluation, training/evaluation at 100,000 gradient steps, and training/evaluation at 200,000 gradient steps). For the inference-time plot, we use the maximum training-compute budget (i.e. $\dbtt=8$).

Figure~\ref{fig:runtime} demonstrates \algo's scalability during both training and inference. At training time, \algo's training runtime can be shortened by decreasing the number of backpropagation steps through time. At inference time, \algo maintains similar inference times compared to \fql under varying inference-time compute budgets.

\newpage
\section{Ablation Studies}
\label{appendix:ablations}

\begin{figure}[H]
  \centering
  \includegraphics[width=0.495\textwidth]{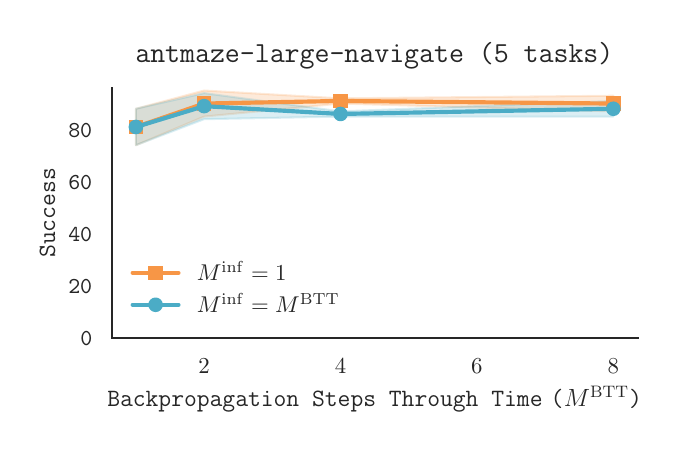}\\
  \includegraphics[width=0.495\textwidth]{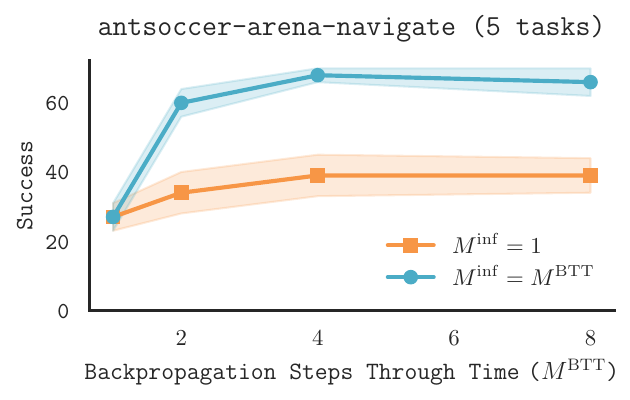}
  \includegraphics[width=0.495\textwidth]{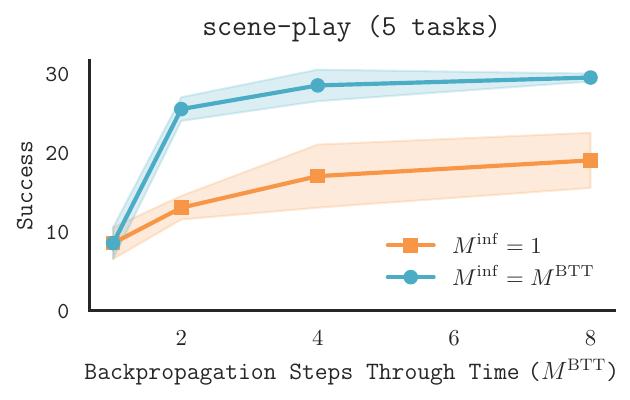}
  \caption{\textbf{Ablation Over Backpropagation Steps Through Time, $\dbtt$.} We investigate the effect of varying the training-time compute budget (i.e. the number of backpropagation steps through time $\dbtt$). The performance is averaged over 8 seeds for each task, with 5 tasks per environment, and standard deviations reported. We report results using one inference step (i.e. $\dinf=1$) and using the same number of inference steps as backpropagation steps through time (i.e. $\dinf=\dbtt$).
  }
  \label{fig:ablation_btt}
\end{figure}

\subsection{Backpropagation Steps Through Time, $\dbtt$}

\textbf{Q: How does increasing training-time compute affect performance?}

\textit{\algo generally observes increasing performance with increased training-time compute (i.e increasing~$\dbtt$), up to a performance saturation point of $\dbtt=4$.}

One of \algo's unique strengths is its scalability at both training-time and inference-time. In this experiment, we consider the question of how much training-time compute is necessary. Based on the results in Figure \ref{fig:ablation_btt}, we see that, in general, performance improves with increased training-time compute (i.e. with an increasing number of backpropagation steps through time, $\dbtt$). However, performance seems to saturate around $\dbtt=4$, suggesting that greater training-time compute may not be necessary for the environments considered. 

The results suggest two key takeaways. First, backpropagation through time over more than one step \textit{is} generally necessary in \algo to maximize performance. Insufficient training-time compute may lead to sub-optimal results, as evidenced in Figure~\ref{fig:ablation_btt} and Tables~\ref{table:offline_table_envs}and\ref{table:offline_table_tasks}, where \algo with $\dbtt \geq 4$ outperforms both the baselines and the versions of \algo with $\dbtt < 4$. Second, the results indicate that \algo does \textit{not} require backpropagation through tens or hundreds of steps---as is common in large diffusion models \citep{ho2020denoising,song2020denoising}---a procedure that could be computationally prohibitive.

\newpage
\subsection{Policy Network Size}
\begin{figure}[t!]
  \centering
  \includegraphics[width=0.60\textwidth]{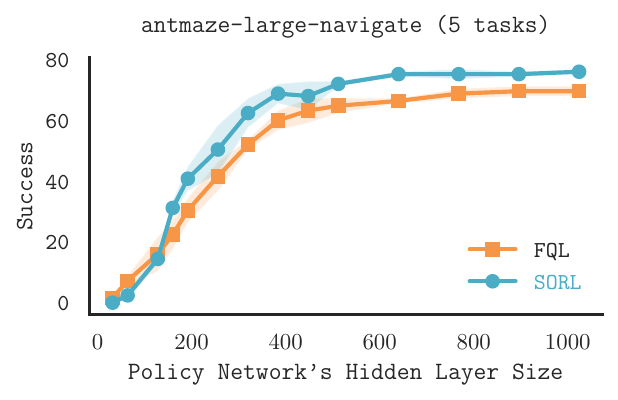}
  \caption{\textbf{Ablation Over Policy Network Size.} The performance is averaged over 5 seeds for each task, with 5 tasks per environment, and standard deviations reported. We use the same training-time and inference-time compute budgets for \algo as we use for Tables \ref{table:offline_table_envs}~and~\ref{table:offline_table_tasks} (i.e. $\dbtt=8$ and $\dinf=4$). We train and evaluate \fql with the parameters used in the official implementation \citep{park2025flow}. The only change to \algo and \fql is varying the sizes of the policy network's hidden layers. 
  }
  \label{fig:ablation_network_size}
\end{figure}

\textbf{Q: How does the performance of \algo vary with changing size of the policy network?}

\textit{Increasing the capacity of the policy network increases the performance of \algo and \fql similarly.}

We use the same training-time and inference-time compute budgets for \algo as we use for Tables \ref{table:offline_table_envs}~and~\ref{table:offline_table_tasks} (i.e. $\dbtt=8$ and $\dinf=4$). We train and evaluate \fql with the parameters used in the official implementation \citep{park2025flow}. The only change to \algo and \fql is varying the sizes of the policy network's hidden layers. We use an MLP with four hidden layers of the same size. Since \fql uses two actor networks of the same size in the official implementation \citep{park2025flow}---a flow-matching network for offline data regularization and a one-step distillation network for the actor's policy---we scale both networks, keeping them the same size for this ablation study.

As may be expected, increasing the capacity of the policy network increases the performance of both \algo and \fql, at similar rates, before plateauing at larger sizes. Notably, \algo achieves slightly better gains across network sizes, which may stem from the shortcut model class being more expressive than \fql's one-step distillation policy. 

\newpage
\section{Experimental and Implementation Details}
\label{appen:impl_det}
In this section, we describe the setup, implementation details, and baselines used in the paper. 

\subsection{Experimental Setup}
The experimental setup for our paper's main results (Table \ref{table:offline_table_envs} and \ref{table:offline_table_tasks}) follows OGBench's official evaluation scheme \citep{park2024ogbench, park2025flow}.

\paragraph{Environments.} 

We now describe the environments, tasks, and offline data used in our experiments. Our setup follows OGBench's official evaluation scheme \citep{park2024ogbench}, which was used by \citet{park2025flow}. We restate the description of environments and tasks below. 

We evaluate \algo on 8 robotics locomotion and manipulation robotics environments in the OGBench task suite (version \texttt{1.1.0}) \citep{park2024ogbench}, a benchmark suite designed for offline RL. Specifically, we use the following environments:
\begin{enumerate}
    \item \texttt{antmaze-large-navigate-singletask-v0}
    \item \texttt{antmaze-giant-navigate-singletask-v0}
    \item \texttt{humanoidmaze-large-navigate-singletask-v0}
    \item \texttt{humanoidmaze-giant-navigate-singletask-v0}
    \item \texttt{antsoccer-arena-navigate-singletask-v0}
    \item \texttt{cube-single-play-singletask-v0}
    \item \texttt{cube-double-play-singletask-v0}
    \item \texttt{scene-play-singletask-v0}
\end{enumerate}

The selected environments span a range of challenging control problems, covering both locomotion and manipulation. The \texttt{antmaze} and \texttt{humanoidmaze} tasks involve navigating quadrupedal (8 degrees of freedom (DOF)) and humanoid (21 DOF) agents through complex mazes. \texttt{antsoccer} focuses on goal-directed ball manipulation using a quadrupedal agent. The \texttt{cube} and \texttt{scene} environments center on object manipulation with a robot arm. Among these, \texttt{scene} tasks require sequencing multiple subtasks (up to 8 per episode), while \texttt{puzzle} emphasizes generalization to combinatorial object configurations. All environments are state-based. We follow the standard dataset protocols (\texttt{navigate} for locomotion, \texttt{play} for manipulation), which are built from suboptimal, goal-agnostic trajectories, and therefore pose a challenge for goal-directed policy learning. Following \citet{park2025flow}, we evaluate agents using binary task success rates (i.e., goal completion percentage), which is consistent with OGBench’s evaluation setup \citep{park2024ogbench}.

\paragraph{Tasks.} Following \citet{park2025flow}, we use OGBench's reward-based \texttt{singletask} variants for all experiments \citep{park2024ogbench}, which are best suited for reward-maximizing RL. As described in the official implementation \citep{park2024ogbench}, each OGBench environment offers five unique tasks, each associated with a specific evaluation goal, denoted by suffixes \texttt{singletask-task1} through \texttt{-task5}. These represent unique, fixed goals for the agent to accomplish. We utilize all five tasks for each environment. The datasets are annotated with a semi-sparse reward, where the reward is determined by the number of remaining subtasks at a given step for manipulation tasks, or whether the agent reaches a terminal goal state for locomotion tasks \citep{park2024ogbench,park2025flow}.

\paragraph{Evaluation.} We follow OGBench's official evaluation scheme \citep{park2024ogbench}, similar to \citet{park2025flow}. We train algorithms for 1,000,000 gradient steps and evaluate 50 episodes every 100,000 gradient steps. We report the average success rates of the final three evaluations (i.e. the evaluation results at 800,000, 900,000, and 1,000,000 gradient steps). In tables, we average over 8 seeds per task and report standard deviations, bolding values within 95\% of the best 
performance.

\newpage
\subsection{\algo Implementation Details}
One of the strengths of \algo is its implementation simplicity. In order to implement \algo, we adapt \citet{park2025flow}'s open-source implementations of various offline RL algorithms (FQL, IFQL, IQL, ReBRAC), which are adapted from \citet{park2024ogbench}'s open-source dataset and codebase. The major changes are implementing and training the shortcut model, and removing additional training complexity from \texttt{FQL} (e.g. the teacher/student networks). To ensure a fair comparison, we keep all major shared hyperparameters the same between the baselines and \algo (e.g. network size, number of gradient steps, and discount factor), unless otherwise noted.

\paragraph{Value Network.} Following \citet{park2025flow}, we train two $Q$ functions and use the mean of the two in the actor update (Equation \ref{eq:q-loss}). Like \citet{park2025flow}, we use the mean of the two $Q$ functions for the critic update (Equation \ref{eq:critic-loss}), except for \texttt{antmaze-large} and \texttt{antmaze-giant} tasks, where we follow \citet{park2025flow} in using the minimum of the two values. The only change to the \texttt{FQL} baseline's value network is adding an input that encodes the number of inference steps used to generate the actions.

\paragraph{Network Architecture and Optimizer.} Following \citet{park2025flow}, we use a multi-layer perceptron with 4 hidden layers of size 512 for both the value and policy networks. Like \citet{park2025flow}, we apply layer normalization \citep{ba2016layer} to value networks. Following \citet{park2025flow}, we use the Adam optimizer \citep{kingma2014adam}, which we add gradient clipping to.

\paragraph{Discretization Steps.} For \algo, we use $\dreg=8$ discretization steps during training (2 fewer than \texttt{FQL}), since \citet{frans2024one} suggests that discretization steps be powers of 2. By default for the results in Tables \ref{table:offline_table_envs} and \ref{table:offline_table_tasks}, we use $\dbtt=8$ steps of backpropagation through time and $\dinf=4$ inference steps, except for \texttt{humanoidmaze-medium} where we use $\dinf=2$. In other experiments, if the values of $\dbtt$ and $\dinf$ are changed, they are explicitly noted.

\paragraph{Shortcut Model.} Following the official implementation of shortcut models \citep{frans2024one}, for the self-consistency loss we sample $d \sim \{2^k\}_{k=0}^{\log_2 M - 1}$ and then sample $t$ uniformly on multiples of $d$ between 0 and 1 (i.e. points where the model may be queried). The model makes a prediction for step size $2d$, and its target is the concatenation of the two sequential steps of size $d$. For greater stability, we construct the targets via a target network. Unlike \citet{frans2024one}, we do not require special processing of the training batch into empirical and self-consistency targets, nor do we require special weight decay. We simply use the entire batch for the critic and actor updates, including the $Q$ loss, self-consistency loss, and flow-matching loss.

\paragraph{Hyperparameters.} We largely use the same hyperparameters for \algo as those used in the \texttt{FQL} baseline (Table \ref{table:shared_params}), except for parameters that are specific to or vary with \algo (Table \ref{table:params_algo}). We hyperparameter tune \algo with a similar training budget to the baselines: we only tune one of \algo's training parameters---\texttt{Q-loss (QL) coefficient}---on the \textit{default} task in each environment (Table \ref{table:params_ql}), over the values \texttt{\{10, 50, 100, 500\}}. We then use the same \texttt{QL coefficient} on all the tasks in an environment. Varying the strength of regularization to the offline data is common in offline RL \citep{tarasov2023corl,park2024value,park2025flow}. Following \citet{park2025flow}'s recommendation, we normalize the Q loss.

\clearpage 
\vspace*{\fill}
\begin{table}[H]
\caption{\label{table:shared_params} Shared Hyperparameters Between \texttt{FQL} Baseline and \algo.}
\centering
\begin{small}
\begin{sc}
\setlength{\tabcolsep}{2pt}
\resizebox{\textwidth}{!}{
\begin{tabular}{ll}
\toprule
Parameter & Value \\
\midrule
Optimizer & Adam \citep{kingma2014adam} \\
Gradient Steps & 1,000,000 \\
Minibatch Size & 256 \\
MLP Dimensions & [512, 512, 512, 512] \\
Nonlinearity & GELU \citep{hendrycks2016gaussian} \\
Target Network Smoothing Coefficient & 0.005 \\
Discount Factor $\gamma$ & 0.99 (\texttt{default}), 0.995 (\texttt{antmaze-giant, humanoidmaze, antsoccer}) \\
Discretization Steps & 8 \\
Time Sampling Distribution & Unif([0,1]) \\
Clipped Double Q-Learning & False (\texttt{default}), True (\texttt{adroit, antmaze-\{large, giant\}-navigate}) \\
\bottomrule
\end{tabular}
}
\end{sc}
\end{small}
\end{table}

\vspace{1cm}
\begin{table}[H]
\caption{\label{table:params_algo} Hyperparameters for \algo.}
\centering
\begin{small}
\begin{sc}
\setlength{\tabcolsep}{2pt}
\begin{tabular}{ll}
\toprule
Hyperparameter & Value \\
\midrule
Learning Rate & 1e-4 \\
Gradient Clipping Norm & 1 \\
Discretization Steps & 8 \\
BC Coefficient & $10$ \\
Self-Consistency Coefficient & $10$ \\
Q-Loss Coefficient & Table \ref{table:params_ql} \\
\bottomrule
\end{tabular}
\end{sc}
\end{small}
\end{table}

\vspace{1cm}
\begin{table}[H]
\caption{\label{table:params_ql} Q-Loss Coefficient for \algo.}
\centering
\begin{small}
\setlength{\tabcolsep}{2pt}
\begin{tabular}{lc}
\toprule
\textsc{Environment} & \textsc{Q Loss Coefficient} \\
\midrule
\texttt{antmaze-large-navigate-v0 (5 tasks)} & 500 \\
\texttt{antmaze-giant-navigate-v0 (5 tasks)} & 500 \\
\texttt{humanoidmaze-medium-navigate-v0 (5 tasks)} & 100 \\
\texttt{humanoidmaze-large-navigate-v0 (5 tasks)} & 500 \\
\texttt{antsoccer-arena-navigate-v0 (5 tasks)} & 500 \\
\texttt{cube-single-play-v0 (5 tasks)} & 10 \\
\texttt{cube-double-play-v0 (5 tasks)} & 50 \\
\texttt{scene-play-v0 (5 tasks)} & 100 \\
\bottomrule
\end{tabular}
\end{small}
\end{table}

\vspace*{\fill}
\clearpage

\subsection{Baselines} 
We evaluate against three Gaussian-based offline RL algorithms (\texttt{BC} \citep{pomerleau1988alvinn}, \texttt{IQL} \citep{kostrikov2021offline}, \texttt{ReBRAC} \citep{tarasov2023revisiting}), three diffusion-based algorithms (\texttt{IDQL} \citep{hansen2023idql}, \texttt{SRPO} \citep{chen2023score}, \texttt{CAC} \citep{ding2023consistency}), and four flow-based algorithms (\texttt{FAWAC} \citep{nair2020awac}, \texttt{FBRAC} \citep{zhang2025energy}, \texttt{IFQL} \citep{wang2022diffusion}, \texttt{FQL} \citep{park2025flow}). For the baselines we compare against in this paper, we report results from \citet{park2025flow}, who performed an extensive tuning and evaluation of the aforementioned baselines on OGBench tasks. We restate descriptions of the baselines here.

\paragraph{Gaussian-Based Policies.} To evaluate standard offline reinforcement learning (RL) methods that employ Gaussian policies, we consider three representative baselines: Behavior Cloning (\texttt{BC})\citep{pomerleau1988alvinn}, Implicit Q-Learning (\texttt{IQL})~\citep{kostrikov2021offline}, and \texttt{ReBRAC}~\citep{tarasov2023revisiting}. \texttt{BC} \citep{pomerleau1988alvinn} serves as a simple imitation learning baseline that directly mimics the demonstration data without any explicit value optimization, while \texttt{IQL} \citep{kostrikov2021offline} represents a popular value-based approach that addresses the overestimation problem in offline RL through implicit learning of the maximum value function. \texttt{ReBRAC} \citep{tarasov2023revisiting}, a more recent method that is known to perform well on many D4RL tasks~\citep{tarasov2023corl}, extends the \texttt{BRAC} framework with regularization techniques specifically designed to constrain the learned policy close to the behavior policy, thereby mitigating the distributional shift problem common in offline RL settings.

\paragraph{Diffusion-Based Policies.} To evaluate diffusion policy-based offline RL methods, we compare to \texttt{IDQL} \citep{hansen2023idql}, \texttt{SRPO} \citep{chen2023score} and Consistency-AC (\texttt{CAC}) \citep{ding2023consistency}. \texttt{IDQL} \citep{hansen2023idql} builds on \texttt{IQL} by using a generalized critic and rejection sampling from a diffusion-based behavior policy. \texttt{SRPO} \citep{chen2023score} replace the diffusion sampling with a deterministic policy trained via score-regularized policy gradient to speed up sampling. \texttt{CAC} \citep{ding2023consistency} introduces a consistency-based actor-critic framework to backpropagation through time with fewer steps.

\paragraph{Flow-Based Policies.} To evaluate flow policy-based offline RL methods, we compare flow-based variants of prior methods, including new variants introduced by \citet{park2025flow}: \texttt{FAWAC} \citep{nair2020awac}, \texttt{FBRAC} \citep{wang2022diffusion}, \texttt{IFQL} \citep{hansen2023idql}, and \texttt{FQL} \citep{park2025flow}, as only a few previous works explicitly employ flow-based policies. \texttt{FAWAC} extends \texttt{AWAC} by adopting a flow-based policy trained with the AWR objective and estimates $Q^\pi$ for the current policy using fully off-policy bootstrapping. \texttt{FBRAC} is the flow counterpart of Diffusion-QL, based on the naïve Q-loss with backpropagation through time. \texttt{IFQL} is a flow-based variant of \texttt{IDQL} that relies on rejection sampling. \texttt{FQL} distills a one-step policy from an expressive flow-matching policy to avoid costly iterative sampling. Unlike \texttt{FQL}, \algo uses a single-stage training procedure and does not require distilling a model into a one-step policy.

\end{document}